\documentclass[twoside]{article}

\usepackage[accepted]{aistats2025}
\usepackage[round]{natbib}

\usepackage{hyperref}
\hypersetup{colorlinks,
linkcolor=blue,
citecolor=blue,
urlcolor=magenta,
linktocpage,
plainpages=false}

\usepackage{amsmath}
\usepackage{amssymb}
\usepackage{amsthm}

\usepackage[algo2e,ruled]{algorithm2e} 

\SetCommentSty{mycommfont}

\newcommand{\la}{\langle}
\newcommand{\ra}{\rangle}
\newcommand{\norm}[1]{\left\lVert#1\right\rVert}
\newcommand{\R}{\mathbb{R}}

\newcommand{\Lag}{\mathcal{L}}
\newcommand{\TLag}{\tilde{\mathcal{L}}}
\newcommand{\td}{\tilde{d}}
\newcommand{\tlambda}{\tilde{\lambda}}
\newcommand{\hdelta}{\hat{\delta}}
\newcommand{\hlambda}{\hat{\lambda}}
\newcommand{\hf}{\hat{f}}
\newcommand{\hx}{\hat{x}}
\newcommand{\T}{\mathcal{T}}
\newcommand{\Ord}{\mathcal{O}}
\newcommand{\Reg}{\mathcal{R}}
\newtheorem{assumption}{Assumption}
\newtheorem{lemma}{Lemma}
\newtheorem{theorem}{Theorem}
\newtheorem{corollary}{Corollary}
\newenvironment{proofsketch}{%
  \proof}{\endproof}

\begin{document}

%

%

\twocolumn[

\runningtitle{Achieving Zero Constraint Violation in Online Learning with Slowly Changing Constraints}

\aistatstitle{Safety in the Face of Adversity: Achieving Zero Constraint Violation in Online Learning with Slowly Changing Constraints}

\aistatsauthor{ Bassel Hamoud \And Ilnura Usmanova \And Kfir Y. Levy }

\aistatsaddress{ ECE Department\\Technion\\Bassel164@campus.technion.ac.il \And SDSC hub\\Paul Scherrer Institute\\ilnura.usmanova@psi.ch \And ECE Department\\Technion\\kfirylevy@technion.ac.il } ]

\begin{abstract}
    We present the first theoretical guarantees for zero constraint violation in Online Convex Optimization (OCO) across all rounds, addressing dynamic constraint changes. Unlike existing approaches in constrained OCO, which allow for occasional safety breaches, we provide the first approach for maintaining strict safety under the assumption of gradually evolving constraints, namely the constraints change at most by a small amount between consecutive rounds. This is achieved through a primal-dual approach and Online Gradient Ascent in the dual space.
    We show that employing a dichotomous learning rate enables ensuring 
    both safety, via zero constraint violation, and sublinear regret.
    Our framework marks a departure from previous work by providing the first provable guarantees for maintaining absolute safety in the face of changing constraints in OCO. 
\end{abstract}
\section{INTRODUCTION}\label{sec:intro}
Online Learning and specifically Online Convex Optimization (OCO) is a fundamental framework towards prediction and sequential decision  making \citep{hazan2023introduction, cesa_bianchi2006}, that has gained increasing popularity due to its ability to capture  non-stationary and even adversarially changing environments. The latter is invaluable in applications in which data evolve over time, such as financial markets, recommender systems, and network security.
Work on OCO typically falls into two categories: work on static regret which compare to some fixed benchmark \citep{zinkevich,Hazan2007_strcvx}, and work on dynamic regret which compare to a changing benchmark and typically obtain bounds that depend on the horizon and problem-dependent quantities, such as the path length of the benchmark or the total variation of the loss functions \citep{zinkevich,Besbes_2015,jadbabaie15_VT,Mokhtari16}. Additionally, there is a separate body of work on \emph{constrained} OCO \citep{Mannor2009,mahdavi_2012,Cao2019,Liu2022,Chen2017}, which strive to balance performance, in term of the regret, and violation of the constraints.

In many such complex scenarios, particularly in real-world applications, performance is not the only consideration, and other aspects like \emph{safety} become paramount. Safety in Machine Learning (ML) is often encapsulated through the concept of safety constraints; that is, rules that the learning process must adhere to in order to avoid undesirable or dangerous outcomes \citep{pmlr-v37-sui15, berkenkamp2020bayesian}. 
Namely, in applications like autonomous driving, where the environment typically changes in a mostly continuous manner, maintaining safety at all times is crucial - never running a red light or endangering pedestrians and ensuring a safe distance between nearby cars.
Yet, traditional approaches to enforcing these constraints often assume a static or stochastic environment, but static throughout time, which leaves a significant gap in our defenses against the unpredictable nature of the real-world. Previous work on constrained OCO has primarily achieved only sublinear bounds on constraint violation, often showcasing a trade-off between regret and constraint violation. Such trade-offs often imply linear regret for small enough (yet larger than zero) violation, which remains unacceptable in safety-critical contexts.

Recognizing this critical vulnerability, our work sets to explore \emph{safe} online learning with \emph{online} constraints. 
Motivated by real-world applications like autonomous cars, where the environment changes dynamically,
our goal is to answer the following theoretical question:
\begin{center}
    \emph{Is it possible to guarantee safety, with zero constraint violation, while maintaining sublinear regret in dynamic environments?}
\end{center}
This shift in focus addresses a significant limitation in existing research on online learning with online constraints, which predominantly concentrates on ensuring sublinear hard constraints violation \citep{Guo_2022} or sublinear long-term constraints violation \citep{Yu_and_Neely_2019}.
Both imply that in individual rounds safety may be violated. 

In our work, we address the challenge of ensuring safety in OCO with dynamically evolving constraints. Recognizing the limitation of existing models in handling non-stationary conditions without compromising safety, we introduce an assumption central to our approach: constraints change gradually over time. This assumption is suitable for dynamic environments and is vital for our model's feasibility, as abrupt changes would render it impossible to establish any meaningful bounds on performance and safety.

Our contribution lies in developing a framework that guarantees zero constraint violation across all learning rounds while ensuring sublinear regret, under the premise of slowly changing constraints. Our work is the first  to provide provable guarantees for safe OCO in such a dynamic setting. By analyzing the learning process against a dynamic comparator sequence, we aim to minimize regret while ensuring each decision made satisfies the evolving constraints, thus maintaining safety in every step. Our emphasis on slow constraint evolution and its capacity to ensure safety sets a new benchmark for research in the domain of safe online learning.
Specifically, 
if changes in  constraint values are restricted by $\delta$ at every time step uniformly over all decision rounds, we show that
one can achieve zero constraint violation and a dynamic regret bound of the form
$\Ord(\sqrt{(V_{g,T}+V_{f,T})T})$ in the strongly convex setting
where $V_{f,T}$ and $V_{g,T}$ are the total variation of the loss functions and the constraints, respectively, over the horizon $T$. 
This is reminiscent of the $\Ord(\sqrt{V_{f,T}T})$ regret bound derived in \cite{Besbes_2015} for the \emph{fixed} constraint setting (with noisy feedback).
Moreover, we generalize our results later in the paper and extend our approach to the convex case, and we show that safety and sublinear dynamic regret can be guaranteed in this setting as well.

On the technical level, 
we show this by devising a novel generalization of the well-known primal-dual approach to the safe OCO scenario.
Much of our analysis is done in the dual space, where we adopt an Online Gradient Ascent (OGA) approach towards choosing the dual variables. Through duality, we are able to analyze both safety and performance: \textbf{(i)} we show that safety can be related to the (online) dual functions, albeit requiring a nonstandard dichotomous learning rate for OGA;
\textbf{(ii)}  we show that the standard dynamic regret of the \emph{loss functions} can be bounded by the dynamic regret of the \emph{dual functions}. Thus, bounding the latter directly translates into guarantees on performance.
 
\paragraph{Related Work.}
Previous work on constrained OCO can be broadly characterized by three key properties, which influence the difficulty of the addressed setting.
(1) \emph{changing vs. fixed constraints}: whether the constraints vary between rounds or are fixed in advance.
(2) \emph{static vs. dynamic regret}: whether the performance is compared to a fixed or a changing comparator.
(3) \emph{hard vs. long-term constraints}: whether strictly feasible decisions compensate for violations.

\cite{mahdavi_2012} studied constrained OCO with \emph{fixed, long-term} constraints and \emph{static} regret and showed a $\Ord(T^{3/4})$ bound on the \emph{long-term} violation and a $\Ord(\sqrt{T})$ bound on the regret.
\cite{jenatton_2016} later examined the same setting and generalized these bounds to $\Ord(T^{1-\beta/2})$ for the long-term violation and $\Ord(T^{\max\{\beta,1-\beta\}})$ for the regret, where $\beta\in(0,1)$ controls the violation-regret trade-off.
\cite{Yu_and_Neely_2019} later improved these violation bounds to $\Ord(T^{1/4})$ while maintaining $\Ord(\sqrt{T})$ regret, and additionally achieved $\Ord(1)$ \emph{long-term} constraint violation under specific additional assumptions.
Particularly, all these works consider fixed,
long-term constraints and static regret. That is, they allow violations to be compensated by strictly feasible decisions and focus solely on bounding the long-term violation. Consequently, these methods do \emph{not} guarantee safety.

In another work, \cite{Yu_and_Neely_2017} established $\Ord(\sqrt{T})$ average violation and static regret for long-term, time-varying constraints, assuming a common feasible set for all constraints. \cite{Cao2019} considered long-term and time-varying constraints, but showed $O(\sqrt{P_T T})$ dynamic regret and $\Ord(T^\frac{3}{4}P_T^\frac{1}{4})$ long-term violation bounds, where $P_T$ is the path length of the dynamic comparator. Similarly to other works, the focus on long-term violation prevents these methods from guaranteeing safety.

\cite{Yuan_Lamperski_2018} studied fixed but \emph{hard} constraints and \emph{static} regret, and showed $\Ord(\sqrt{T\log(T)})$ violation and $\Ord(\log(T))$ regret in the strongly convex setting.
Later, \cite{Yi_Johansson_2021} considered a similar setting, although with \emph{dynamic} regret, and achieved $\Ord(\sqrt{T})$ violation and $\Ord(\sqrt{T(1+P_T)})$ regret, where $P_T$ is the path length of the dynamic comparator.
\cite{Guo_2022} addressed the hardest setting among these works, specifically \emph{changing, hard} constraints and \emph{dynamic} regret, and established $\Ord(\sqrt{T\log(T)})$ violation and $\Ord(P_T\sqrt{T})$ regret in the strongly convex setting.
\cite{Kolev} devised a velocity projection method that guarantees $\Ord(\sqrt{T})$ \emph{static} regret and a maximum violation of $\Ord(1/\sqrt{t})$ per round, assuming constraints change by at most $\Ord(1/t)$ between rounds, along with additional assumptions on the feasible set.
Notably, in these works, constraints may still be violated despite disallowing compensation through strictly feasible decisions. Therefore, these methods also cannot guarantee safety.

Concurrent to our work, \cite{Hutchinson_safety_2024} demonstrated that $\Ord(\sqrt{T(P_T+1}))$ regret and zero constraint violation can be guaranteed under \emph{strongly} convex and \emph{monotone} constraints, i.e., when the feasible sets satisfy $\mathcal{X}_1 \subseteq \mathcal{X}_2 \subseteq...\subseteq \mathcal{X}_T$. In contrast, our work assumes only convex constraints that change gradually, without requiring strong convexity.
In particular, monotone constraints significantly simplify safety enforcement, as a feasible decision in any round remains feasible in all subsequent rounds. Conversely, in our setting, feasible decisions in one round may become infeasible in the next round, inducing a more complex scenario in which the chosen action must continuously adapt to ensure safety.

In stark contrast to prior work on constrained OCO, 
we provide the first theoretical guarantees on both \emph{zero} constraint violation, ensuring safety, and sublinear dynamic regret.
Moreover, we address the most difficult setting, with changing hard constraints and dynamic regret, assuming slowly evolving constraints.

\section{PROBLEM STATEMENT}
We consider the task of safe online optimization with a slowly changing constraint and horizon $T$. That is, in each iteration $t\in[T]$ the learner chooses an action $x_t\in\mathcal{X}$, where $\mathcal{X} \subset \R^D$ is a convex and bounded action set. Then, \emph{after} $x_t$ is chosen, the loss function $f_t:\mathcal{X} \rightarrow \R$ and constraint $g_t:\mathcal{X} \rightarrow \R$ are revealed, and the learner suffers the corresponding loss $f_t(x_t)$ and violation $g_t(x_t)$. We measure the performance of the learner in terms of the dynamic regret:
\begin{equation*}\label{eq:orignal_online_problem}
    \Reg_f(T) = \sum_{t=1}^T{f_t(x_t) - f_t(x_t^*)},
    \tag{P}
\end{equation*}
where the dynamic comparator sequence is defined as follows, for any $t\in[T]$:
\begin{equation}\label{eq:primal_baseline}
    x_t^* = \arg\min_{x\in\mathcal{X}} f_t(x) \quad \text{s.t.} \quad g_t(x) \leq 0.
\end{equation}
Our goal is to achieve sublinear regret while satisfying the constraints in every step, i.e., keeping the constraint violation identically \emph{zero}, or equivalently $g_t(x_t) \leq 0, \forall t\in[T]$. Note that $x_t^*$ is the best possible action at step $t$ as it attains the smallest loss subject to the constraint. Although a dynamic comparator sequence is more challenging, it is necessary since the comparator must satisfy the \emph{changing} constraints in every step. This is impossible to demand from a static comparator without additional exorbitant assumptions.
Moreover, note that since $x_t^*$ is the optimal comparator sequence, it is, by definition, the "most challenging" comparator sequence, in the sense that guaranteeing sublinear regret w.r.t $\{x_t^*\}_{t=1}^T$ ensures the same guarantees for \emph{any} comparator sequence $\{u_t\}_{t=1}^T$. That is because for any $\{u_t\}_{t=1}^T$ such that $g_t(u_t)\leq0,\forall t\in[T]$, we have: $\sum_{t=1}^T f_t(x_t) - f_t(u_t) \leq \sum_{t=1}^T f_t(x_t) - f_t(x_t^*)$. Thus, obtaining sublinear regret w.r.t $x_t^*$ is a stronger result.

\paragraph{Notation and Definitions.}
We denote the feasible set defined by the constraint $g_t(x)$ by $\mathcal{X}_t$, namely $\mathcal{X}_t := \{x\in\mathcal{X} : g_t(x) \leq 0\}$, and its interior by $\text{Int}(\mathcal{X}_t)$.
Additionally, we denote the $\ell_2$-norm by $\norm{\cdot}$, define $[T] := \{1,2,...,T\}$, and use the "little o" notation as follows: $f(x) = o(g(x))$ if $\lim_{x\rightarrow\infty} \frac{f(x)}{g(x)} \rightarrow 0$.
A function $f:\mathcal{X} \rightarrow \R$ is $\mu$-strongly convex if $\forall x,y\in\mathcal{X}$: 
\begin{equation*}
    f(y) \geq f(x) + \langle \nabla f(x), y-x \rangle + \frac{\mu}{2}\norm{y-x}^2,
\end{equation*}
it is $M$-smooth if $\forall x,y\in\mathcal{X}$: 
\begin{equation*}
    f(y) \leq f(x) + \langle \nabla f(x), y-x \rangle + \frac{M}{2}\norm{y-x}^2,
\end{equation*}
and it is $L$-Lipschitz continuous if $\forall x,y\in\mathcal{X}$: \begin{equation*}
    |f(y) - f(x)| \leq L\norm{y-x}.
\end{equation*}

Constrained optimization problems of the form $\min_{x\in\mathcal{X}} f(x)\; \text{s.t.}\; g(x)\leq 0$, can be written as $\min_{x\in \mathcal X} \max_{\lambda\geq 0} \Lag (x,\lambda), $
where $\Lag (x,\lambda) := f(x) + \lambda g(x)$ is the Lagrangian and $\lambda \geq 0$ is the dual variable. 
The corresponding dual function is defined by $d(\lambda) := \min_{x\in\mathcal{X}} \Lag(x,\lambda)$ and its optimal dual variable by $\lambda^* := \arg\max_{\lambda \geq 0} d(\lambda)$. Thus, for a problem at time step $t$ given by
$\min_{x\in\mathcal{X}} f_t(x)\; \text{s.t.}\; g_t(x)\leq 0,$
we denote the corresponding Lagrangian, dual function, and dual optimum by $\Lag_t (x,\lambda)$, $d_t(\lambda)$, and $\lambda_t^*$, respectively.  Similarly, we define the danger-aware optimization problem with a shrunk constraint as follows 
$\min_{x\in\mathcal{X}} f_t(x) \;\text{s.t.}\; g_t(x)+\delta \leq 0.$
We denote the corresponding Lagrangian, dual function, and dual optimum by $\TLag_t(x,\lambda)$, $\td_t(\lambda)$, and $\tlambda_t^*$, respectively.
Finally, we denote the optimal value of $\Lag_t$ and $\TLag_t$ over $x$ for a specific $\lambda\geq0$ by:
\begin{equation}\label{eq:x_t_lambda_opt}
    x_{t,\lambda}^* = \arg\min_{x\in\mathcal{X}} \Lag_t(x,\lambda) = \arg\min_{x\in\mathcal{X}} \TLag_t(x,\lambda).
\end{equation}

\paragraph{Assumptions.}
We make the following assumptions throughout the paper:
\begin{assumption}\label{assum:bounded_set}
    The action set $\mathcal{X}$ is simple (e.g., a $d$-dimensional Euclidean ball), convex, and bounded: $\exists R>0: \forall x,y\in\mathcal{X}, \norm{x-y} \leq R$, and the feasible sets $\mathcal{X}_t$ are convex and contained in $\mathcal{X}$: $\mathcal{X}_t \subset \mathcal{X}, \forall t\in[T]$.
\end{assumption}
\begin{assumption}\label{assum:obj_smooth_strconv_lipsc}
    The loss functions $f_t(x), \forall t \in [T]$, are $\mu$-strongly convex, $M_f$-smooth, and $L_f$-Lipschitz continuous over $\mathcal{X}$ w.r.t the $\ell_2$-norm.
\end{assumption}
\begin{assumption}\label{assum:constr_conv_lipsc}
    The constraints $g_t(x)$, $\forall t\in[T]$, are convex, $M_g$-smooth, and $L_g$-Lipschitz continuous over $\mathcal{X}$ w.r.t. the $\ell_2$-norm.
\end{assumption}

\begin{assumption}\label{assum:slowly_changing_constr}
    The constraints $g_t(x)$ change $\delta$-slowly between consecutive time steps, with $\delta \geq 0$: $\forall t\in\{2,3,...,T\}: \max_{x\in\mathcal{X}}|g_t(x) - g_{t-1}(x)| \leq \delta$.
\end{assumption}
We allow $\delta$ to depend on the horizon $T$. Without this assumption, a large abrupt change in the constraints may make safety impossible to guarantee.
This assumption implies that the total variation of the constraints $\sum_{t=2}^T \max_{x\in\mathcal{X}}|g_t(x)-g_{t-1}(x)|$ is bounded by $V_{g,T}=\delta T$. In this paper, we show that sublinear regret necessitates $V_{g,T} = o(T)$, which is implied by $\delta = o(T^{-\alpha})$, with $\alpha>0$.
Similar settings have been considered in previous work, e.g., \cite{Kolev} assumes $\|g_t-g_{t-1}\|_{\infty} = \Ord(1/t)$. Ours is a similar but more general assumption.

\begin{assumption}\label{assum:bounded_TV_obj}
    The loss functions $f_t(x)$ have bounded total variation $V_{f,T}$ of the following form:\\ $\sum_{t=2}^{T}{\max_{x\in\mathcal{X}} |f_t(x) - f_{t-1}(x)}| \leq V_{f,T}$.
\end{assumption}
Here, we allow $V_{f,T}$ to depend on $T$, and specifically require $V_{f,T}=o(T)$ as in \citet{Besbes_2015, Besbes_bandits_2014, jadbabaie15_VT}.
Moreover, inspired by these works, this assumption on $V_{f,T}$ makes it possible to obtain sublinear regret w.r.t the \emph{best} possible comparator $x_t^*$, defined in Eq.~(\ref{eq:primal_baseline}), which directly implies sublinear regret w.r.t \emph{any} comparator sequence.

\begin{assumption}\label{assum:non_shallow_constr_and_strong_duality}
    There exists a positive constant $G$ such that $\forall t\in[T]$, $\exists x_t^0 \in \mathcal{X}_t : g_t(x_t^0) \leq -G$.
\end{assumption}
This assumption implies that the constraints are not "too shallow". It also implies Slater's condition, and thus, since the optimization problem is convex, strong duality holds for any $t\in[T]$.
\begin{assumption}\label{assum:safe_starting_point}
    There exists a known safe starting point $x_1\in\mathcal{X}$ such that $g_1(x_1) \leq 0$.
\end{assumption}
Without this assumption, since the constraints are unknown a priori, safety would be impossible to guarantee since even the first point might not be safe. This is a standard assumption in safe optimization literature \citep{
berkenkamp2020bayesian,
usmanova2023log}.

\paragraph{Preliminaries.}
We show two helpful lemmas which prove useful throughout the paper. Please refer to Appendix 
\ref{appendix:universal_bound_lambda} and \ref{appendix:value_of_mu_d} for the proofs.
\begin{lemma}\label{lemma:universal_bound_dual_baseline}
    Under Assumptions \ref{assum:bounded_set}-\ref{assum:obj_smooth_strconv_lipsc} and \ref{assum:non_shallow_constr_and_strong_duality}, the optimal dual values $\lambda_t^* = \arg\max_{\lambda \geq 0} d_t(\lambda)$ and $\tlambda_t^* = \arg\max_{\lambda \geq 0} \td_t(\lambda)$ are bounded by $\hat{\lambda} = \frac{L_fR}{G}$, $\forall t\in[T]$, namely, $\lambda_t^* \leq \hat{\lambda}$ and $\tlambda_t^* \leq \hat{\lambda}$, $\forall t\in[T]$.
\end{lemma}
Note that since the dual functions $d_t(\lambda)$ and $\td_t(\lambda)$ are one-dimensional and concave $\forall t\in[T]$, their gradients $\nabla d_t(\lambda)$ and $\nabla \td_t(\lambda)$, respectively, are monotonically non-increasing. 
\begin{lemma}\label{lemma:local_str_conc_of_dual}
    Under Assumptions \ref{assum:bounded_set}-\ref{assum:constr_conv_lipsc}, \ref{assum:non_shallow_constr_and_strong_duality},
    $d_t(\lambda)$ and $\td_t(\lambda)$ are locally $\mu_d$-strongly concave, $\forall t\in[T]$, with $\mu_d = \frac{G^2}{4R^2(M_f + \hat{\lambda}M_g)}$,
    $\forall \lambda: g_t(x_{t,\lambda}^*) \geq -G/2$.
\end{lemma}

\section{OUR APPROACH}
\subsection{Warm Up}
As a warm up, and to provide initial intuition, let us assume access to a strong optimization oracle that solves constrained optimization problems, as follows:
\begin{equation*}
    O_{S}(f,g) = \arg\min_{x\in\mathcal{X}} f(x) \quad\text{s.t.}\quad g(x) \leq 0.
\end{equation*}
Moreover, we assume that $O_S$ returns the primal-dual solution, $(x^*,\lambda^*)$. Note that while such an oracle is impractical to use, it helps to provide valuable intuition about the nature of our problem, which will be helpful later.
To this end, we introduce Alg.~\ref{alg:naive} as the na\"ive approach for the safe online problem (\ref{eq:orignal_online_problem}). Recall that $f_t$ and $g_t$ are \emph{unknown} prior to choosing $x_t$, and that, by Assumption \ref{assum:safe_starting_point}, there exists a known safe starting point $x_1$ such that $g_1(x_1) \leq 0$.
\begin{algorithm2e}\label{alg:naive}
  \DontPrintSemicolon
  \KwData{Horizon $T$}
  \textbf{Initialization:} safe starting point $x_1$\;
  \For {$t=2,3,...,T+1$}{
        Play: $x_{t-1}$\;
        Suffer: $f_{t-1}(x_{t-1}), g_{t-1}(x_{t-1})$\;
        Update: $x_t \gets O_{S}(f_{t-1}, g_{t-1}+\delta)$
    }
  \caption{Safe Na\"ive Algorithm}
\end{algorithm2e}
\begin{theorem}\label{theorem:guarantees_naive_alg}
    Consider a safe online optimization problem with horizon $T$ of the form (\ref{eq:orignal_online_problem}). Under Assumptions~\ref{assum:bounded_set}-\ref{assum:safe_starting_point},
    Alg.~\ref{alg:naive} guarantees zero constraint violation and the following sublinear dynamic regret w.r.t the comparator sequence defined in Eq.~(\ref{eq:primal_baseline}):
    \begin{equation*}
       \Reg_f(T) = \Ord\left(\sqrt{(V_{f,T} + V_{g,T} )T}\right). 
    \end{equation*}
    Please refer to Appendix.~\ref{appendix:full_proof_naive_alg_guarantees} for the proof.
\end{theorem}

\subsection{A More Efficient Approach}
\paragraph{Weak Optimization Oracle and Dual Regret.}
Building on Alg.~\ref{alg:naive} and Theorem \ref{theorem:guarantees_naive_alg},
we introduce a more efficient approach that leverages a more practical, weaker oracle to address the safe online problem (\ref{eq:orignal_online_problem}). We show that sublinear regret and zero constraint violation, i.e. safety, can still be guaranteed.
The gain in efficiency is gained by using the following weaker optimization oracle instead of the strong oracle:
\begin{equation}\label{eq:unconstrained_oracle}
    O_{unc}(f,g, \lambda) = \arg\min_{x\in\mathcal{X}} f(x) + \lambda g(x).
\end{equation}
Since the set $\mathcal{X}$ is simple, projection onto $\mathcal{X}$ is computationally inexpensive. 
Consequently, this oracle returns the solution to a much simpler, nearly \emph{unconstrained} 
optimization problem. This is in contrast to the strong oracle which solves a more complex \emph{constrained} optimization problem with functional constraints which often induce complex feasibility sets with costly projection operations.

Given this more practical weaker oracle, we propose a novel dual approach for constrained online learning. Specifically, we define the danger-aware \emph{dual regret}, over a sequence of dual decisions $\{\lambda_t\}_{t=1}^T$, as the regret in terms of the danger-aware dual functions corresponding to each step, namely:
\begin{equation}\label{eq:dual_reg_def}
    \Reg_{\td}(T) = \sum_{t=1}^{T} \td_t(\tlambda_t^*) - \td_t(\lambda_t)
\end{equation}
where $\tlambda_t^* = \max_{\lambda \geq 0} \td_t(\lambda)$, for $t\in[T]$, is the optimal danger-aware dual comparator sequence. Recall that the danger-aware dual function is defined as:
\begin{equation}\label{eq:dual_def_with_x_t_lambda}
    \td_t(\lambda) = \min_{x\in\mathcal{X}} \TLag_t(x,\lambda) = f_t(x_{t,\lambda}^*) + \lambda (g_t(x_{t,\lambda}^*)+\delta)
\end{equation}
where $x_{t,\lambda}^*$ is the minimizer of $\TLag_t(x,\lambda)$ over $x$ for a given $\lambda$. 
Note that this definition of regret in Eq.~(\ref{eq:dual_reg_def}) differs slightly from the conventional one as the comparator sequence here appears in the first term. This distinction arises because we seek to maximize the dual functions, unlike in the standard setting, where the objective is to minimize the loss functions.

This novel dual approach, while nonstandard in online learning, is key for ensuring safety. Dynamic regret under changing constraints is not well explored. To the best of our knowledge, all previous works in this field allow some degree of constraint violation, provided that the total or net violation is sublinear in $T$ (see Sec.~\ref{sec:intro}).
Safety, however, by satisfying the constraints in \emph{every} step, is a much more stringent requirement. We exploit duality to transform the primal problem (\ref{eq:orignal_online_problem}) with \emph{changing} constraints into a dual problem with a \emph{fixed} simple constraint over the dual variables. This makes safety easier to guarantee as we show in Lemma \ref{lemma:safety_in_terms_of_dual} later in the paper.

In the rest of the paper, we introduce a novel algorithm that exploits the benefits of duality in Alg.~\ref{alg:safe_dual_alg_warm_start}, and show that it achieves zero constraint violation in Theorem \ref{theorem:safe_stepsize_bounds}. To show the regret guarantees, we first relate the \emph{dual} regret, $\Reg_{\td}(T)$, to the \emph{primal} regret, $R_f(T)$, defined in (\ref{eq:orignal_online_problem}), in Lemma \ref{lemma:primal_dual_regret_relation}. Then finally, in Theorem \ref{theorem:regret_analyis_our_method} we bound the dual regret of Alg.~\ref{alg:safe_dual_alg_warm_start} and use the relation between the primal and dual regret to obtain a sublinear bound on the primal regret.

\paragraph{The Main Idea.}
Our approach assumes access to the strong oracle only \emph{once} during the initialization of our algorithm.
This can be regarded as a "warm start", which moves the algorithm's dual iterates closer to the dual optimal values, thereby enabling tracking them early.
After the first iteration, our approach relies solely on the weaker oracle defined in Eq.~(\ref{eq:unconstrained_oracle}).
We show that this approach guarantees the same regret bounds as Alg.~\ref{alg:naive} while being more efficient due to the use of the weaker oracle instead of the strong oracle.
To achieve this, we propose a safe online dual gradient ascent approach in Alg.~\ref{alg:safe_dual_alg_warm_start} that aims to minimize the dual regret while ensuring zero constraint violation.

\begin{algorithm2e}\label{alg:safe_dual_alg_warm_start}
  \DontPrintSemicolon
  \KwData{Horizon T}
  \textbf{Initialization:} safe starting point $x_1$ and dual counterpart $\lambda_1$ using the strong oracle\;
  \For {$t=2,3,,...,T+1$}{
        Play: $x_{t-1}$\;
        Suffer: $f_{t-1}(x_{t-1}), g_{t-1}(x_{t-1})$\;
        Update:
        $x_{t-1,\lambda_{t-1}}^* \gets O_{unc}(f_{t-1},g_{t-1},\lambda_{t-1})$\;        
        \quad\quad\quad\,\,\,$\lambda_t \gets \Pi_{\R_+}(\lambda_{t-1} + \gamma_t \nabla\tilde{d}_{t-1}(\lambda_{t-1}))$\tcp*{$\nabla\tilde{d}_{t-1}(\lambda_{t-1}) = g_{t-1}(x_{t-1,\lambda_{t-1}}^*) + \delta$}
        \quad\quad\quad\,\,\,$x_t \gets O_{unc}(f_{t-1},g_{t-1},\lambda_t)$
  }
  \caption{Safe Online Dual Gradient Ascent}
\end{algorithm2e}

In particular, Alg.~\ref{alg:safe_dual_alg_warm_start} requires solving two "primal" subproblems in each iteration. This can be efficiently handled using any off-the-shelf optimization method.
Next, we analyze the safety and regret of Alg.~\ref{alg:safe_dual_alg_warm_start} and show that it guarantees zero constraint violation (safety) and sublinear regret.

\subsection{Safety in Online Learning with Changing Constraints}
\paragraph{The Main Principle.}
Algorithm \ref{alg:safe_dual_alg_warm_start} can be viewed as a variant of online gradient ascent in the dual space, which performs only a single dual gradient step at a time. This is in contrast to Alg.~\ref{alg:naive} which has full access to a strong oracle and can \emph{fully} solve the dual problem in every iteration.
Since the gradient steps occur in one-dimensional space, they will either decrease or increase the iterates $\lambda_t$, depending on the sign of $\nabla \td_{t-1}(\lambda_{t-1})$.
When $\nabla \td_{t-1}(\lambda_{t-1}) \leq 0$, we have $g_{t-1}(x_{t-1,\lambda_{t-1}}^*)+\delta < 0$, by definition of the dual gradient. This implies that $x_{t-1,\lambda_{t-1}}^*$ is safe and sufficiently far from the boundary and thus $\lambda_{t-1}$ can be \emph{decreased}.
Conversely, when $\nabla \td_{t-1}(\lambda_{t-1}) > 0$, we have $g_{t-1}(x_{t-1,\lambda_{t-1}}^*)+\delta > 0$ which implies that $x_{t-1,\lambda_{t-1}}^*$ is safe but dangerously close to the boundary and thus $\lambda_{t-1}$ must be \emph{increased} to push it to a safer region.
Thus, the dual updates naturally switch between two modes (or phases): \emph{safe phases}, where $\lambda_{t-1}$ is decreased, and \emph{danger phases}, where $\lambda_{t-1}$ is increased.
These phases require different step size restrictions for safety, as we demonstrate next.

Since Alg.~\ref{alg:safe_dual_alg_warm_start} does not fully solve the dual problem, guaranteeing safety is more challenging.
To address this, we first establish a key lemma that reformulates the safety condition on $x_t$, namely $g_t(x_t) \leq 0$, in terms of the dual iterates $\lambda_t$.
\begin{lemma}\label{lemma:safety_in_terms_of_dual}
    Under Assumptions \ref{assum:slowly_changing_constr} and \ref{assum:non_shallow_constr_and_strong_duality}, for any step $t\in\{2,3,...,T\}$, having $\nabla \td_{t-1}(\lambda_t) \leq 0$ ensures that the iterates $x_t$ of Alg.~\ref{alg:safe_dual_alg_warm_start} are safe, namely $g_t(x_t) \leq 0$.
\end{lemma}
\begin{proof}
    We have:
    \begin{align*}
        g_t(x_t) \leq g_{t-1}(x_t) + \delta
        = g_{t-1}(x_{t-1,\lambda_t}^*) + \delta
        = \nabla \td_{t-1}(\lambda_t),
    \end{align*}
    where the inequality follows by Assumption \ref{assum:slowly_changing_constr}, the first equality follows since $x_t = x_{t-1,\lambda_t}^*$ by Alg.~\ref{alg:safe_dual_alg_warm_start} and Eq.~(\ref{eq:x_t_lambda_opt}), and the last equality follows by the definition of the dual function (Eq.~(\ref{eq:dual_def_with_x_t_lambda})).
    Thus, for any $t\in[T]$, safety, namely $g_t(x_t) \leq 0$, can be guaranteed by updating $\lambda_t$ such that $\nabla \td_{t-1}(\lambda_t) \leq 0$.
\end{proof}

Now, since $\nabla \td_t(\lambda)$ is monotonically non-increasing, $\forall t\in[T]$, as the gradient of a concave function, the safety criterion $\nabla \td_{t-1}(\lambda_t) \leq 0$ induces two distinct behaviors of
the dual update rule for $\lambda_t$ in Alg.~\ref{alg:safe_dual_alg_warm_start}, depending on the sign of $\nabla \td_{t-1}(\lambda_{t-1})$:
\begin{itemize}
    \item[(1)] \textbf{The safe phase:} $\nabla \td_{t-1}(\lambda_{t-1}) \leq 0$. Here, we can \emph{decrease} the dual variable as long as $\nabla \td_{t-1}(\lambda_t) \leq 0$. This induces an \emph{upper} bound on the step size $\gamma_t$ to ensure safety.
    \item[(2)] \textbf{The danger phase:} $\nabla \td_{t-1}(\lambda_{t-1}) > 0$. Here, we must \emph{increase} the dual variable sufficiently to ensure that $\nabla \td_{t-1}(\lambda_t) \leq 0$. This induces a \emph{lower} bound on $\gamma_t$ to ensure safety.
\end{itemize}
In stark contrast to standard online learning literature in which the step size (learning rate) is typically constant or monotonically decreasing, this dichotomy in behavior gives rise to a nonstandard dichotomous step size, as we show next in Theorem \ref{theorem:safe_stepsize_bounds}.

\subsubsection{Safety Guarantees} Before we derive the upper and lower bounds on the step size for each case and construct the dichotomous learning rate, we show a helpful property of the dual functions (see proof in App. \ref{appendix:proof_corollary_lipsch_of_dual}).

\begin{corollary}\label{corollary:lipschitz_cont_of_dual}
    The gradient of the dual function $\nabla \td_t(\lambda)$ is $L_g^2/\mu$-Lipschitz continuous, and accordingly the dual function $\td_t(\lambda)$ is $L_g^2/\mu$-smooth, $\forall t\in[T]$.
\end{corollary}
Now, we derive the upper and lower safety bounds on the step size $\gamma_t$ in Alg.~\ref{alg:safe_dual_alg_warm_start}.
\begin{theorem}\label{theorem:safe_stepsize_bounds}
    Under Assumptions \ref{assum:bounded_set}-\ref{assum:non_shallow_constr_and_strong_duality}, Alg.~\ref{alg:safe_dual_alg_warm_start} with $\gamma_t \leq \mu/L_g^2$ when $\nabla\tilde{d}_{t-1}(\lambda_{t-1}) \leq 0$ and $\gamma_t \geq 2/\mu_d$ when $\nabla\tilde{d}_{t-1}(\lambda_{t-1}) > 0$ guarantees $g_t(x_t) \leq 0, \forall t\in[T]$.
\end{theorem}
\begin{proofsketch}
    For any $t\in[T]$, we split the proof according to the sign of $\nabla \td_{t-1}(\lambda_{t-1})$: \\
    \textbf{The safe phase:} $\nabla\tilde{d}_{t-1}(\lambda_{t-1}) \leq 0$. Note that:
    \begin{align*}
        |\nabla\tilde{d}_{t-1}(\lambda_t) - \nabla\tilde{d}_{t-1}(\lambda_{t-1})|
        &\leq\frac{L_g^2}{\mu} |\lambda_t - \lambda_{t-1}|\\
        &\leq \frac{L_g^2}{\mu}\gamma_t |\nabla\tilde{d}_{t-1}(\lambda_{t-1})|,
    \end{align*}
    where the first inequality follows by Corollary \ref{corollary:lipschitz_cont_of_dual} and the second by the dual update in Alg.~\ref{alg:safe_dual_alg_warm_start}.
    Thus, we can ensure $\nabla\tilde{d}_{t-1}(\lambda_t) \leq 0$ by choosing $\gamma_t$ such that
    \begin{align*}
        \nabla\tilde{d}_{t-1}(\lambda_{t-1}) + \frac{L_g^2}{\mu}\gamma_t |\nabla\tilde{d}_{t-1}(\lambda_{t-1})| \leq 0,
    \end{align*}
    which induces the following upper bound on $\gamma_t$ (recall that $\nabla \td_{t-1}(\lambda_{t-1}) \leq 0$):
    \begin{align*}
        \gamma_t \leq \mu/L_g^2.
    \end{align*}

    \textbf{The danger phase:} $\nabla\tilde{d}_{t-1}(\lambda_{t-1}) > 0$.
    Note that the dual function $\td_t(\lambda)$ is concave, $\forall t\in[T]$. Moreover, it is locally $\mu_d$-strongly concave, $\forall t\in[T]$, by Lemma \ref{lemma:local_str_conc_of_dual}. Thus, for step $t-1$ and for $\lambda = \lambda_{t-1}$ we have:
    \begin{align*}
         \langle \nabla \td_{t-1}(\lambda_{t-1}), \tlambda_{t-1}^* - \lambda_{t-1} \rangle &\geq \td_{t-1}(\tlambda_{t-1}^*) - \td_{t-1}(\lambda_{t-1}) \\
         &\geq \frac{\mu_d}{2} (\tlambda_{t-1}^* - \lambda_{t-1})^2.
    \end{align*}
    Now, note that since $\nabla \td_{t-1}(\lambda)$ is monotonically non-increasing and since $\nabla \td_{t-1}(\tlambda_{t-1}^*) = 0$ and $\nabla \td_{t-1}(\lambda_{t-1}) > 0$, we have that $\lambda_{t-1} \leq \tlambda_{t-1}^*$. 
    Thus, dividing by $\tlambda_{t-1}^* - \lambda_{t-1}$ and rearranging, we have: $\tlambda_{t-1}^*  \leq \lambda_{t-1} + \frac{2}{\mu_d}\nabla\td_{t-1}(\lambda_{t-1})$.
    Now, by Lemma \ref{lemma:safety_in_terms_of_dual}, to ensure safety, $\lambda_t$ must be chosen such that $\nabla\td_{t-1}(\lambda_t) \leq 0$. This is equivalent to choosing $\lambda_t \geq \tlambda_{t-1}^*$ since $\nabla\td_{t-1}(\lambda)$ is monotonically non-increasing and $\nabla\td_{t-1}(\tlambda_{t-1}^*) = 0$. Using the dual update rule in Alg.~\ref{alg:safe_dual_alg_warm_start}, we choose $\lambda_t$ such that:
    \begin{align*}
        \lambda_t &= \lambda_{t-1} + \gamma_t \nabla\td_{t-1}(\lambda_{t-1})\\
        &\geq \lambda_{t-1} + \frac{2}{\mu_d}\nabla\td_{t-1}(\lambda_{t-1})
        \geq \tlambda_{t-1}^*,
    \end{align*}
    which, since $\nabla \td_{t-1}(\lambda_{t-1}) > 0$, induces the following lower bound: $ \gamma_t \geq \frac{2}{\mu_d}.$
Please refer to Appendix \ref{appendix:full_proof_stepsize_bounds} for the full proof.
\end{proofsketch}
To conclude, we derived a nonstandard dichotomous bound on the step size $\gamma_t$, determined by the sign of $\nabla \td_{t-1}(\lambda_{t-1})$ in each step $t$. Notably, the bounds in Theorem \ref{theorem:safe_stepsize_bounds} satisfy $2/\mu_d \geq \mu/L_g^2$. This follows from Corollary \ref{corollary:lipschitz_cont_of_dual}, which establishes that $L_g^2/\mu$ is an upper bound on the curvature of the dual function $\td_t(\lambda)$, while Lemma \ref{lemma:local_str_conc_of_dual} shows that $\mu_d$ is a lower bound on its curvature. This means that there is \emph{no} constant $\gamma_t$ that satisfies both bounds simultaneously, and it must adapt to align with the appropriate bound corresponding to each phase. This fundamental property gives rise to the dichotomous step size for ensuring safety.

\subsection{Bounding the Regret}
To analyze the regret of Alg.~\ref{alg:safe_dual_alg_warm_start}, we relate the primal regret $\Reg_f(T)$ and the dual regret $\Reg_{\td}(T)$ and show two useful properties. Please refer to Appendix \ref{appendix:primal_dual_regret_relation}, \ref{appendix:slowly_changing_dual_gradients}, and \ref{appendix:bouned_distance_between_lambda_tlambda} for the proofs.
\begin{lemma}\label{lemma:primal_dual_regret_relation}
    Given a safe online problem~(\ref{eq:orignal_online_problem}) with horizon $T$, under assumptions \ref{assum:bounded_set}-\ref{assum:safe_starting_point}, and given an upper bound $\hat{\Reg}_{\td}(T)$ on the dual regret $\Reg_{\td}(T)$ defined in Eq.~(\ref{eq:dual_reg_def}), the primal regret $\Reg_f(T)$ defined in (\ref{eq:orignal_online_problem}) suffered by Alg.~\ref{alg:safe_dual_alg_warm_start} is bounded as follows:
    \begin{equation*}
        \Reg_f(T) = \Ord\left(\max\left\{\hat{\Reg}_{\td}(T), \sqrt{(V_{f,T}+V_{g,T})T}\right\}\right)
    \end{equation*}
\end{lemma}
\begin{lemma}\label{lemma:slowly_changing_dual_gradients}
    Under Assumptions \ref{assum:bounded_set}-\ref{assum:slowly_changing_constr} and \ref{assum:non_shallow_constr_and_strong_duality}, the distance between consecutive dual gradients is bounded as: $\max_{\lambda\geq0}|\nabla \td_t(\lambda) - \nabla \td_{t-1}(\lambda)| \leq \hdelta_t$, with:
    \begin{equation*}
        \hdelta_t = \delta + L_g\sqrt{\frac{2}{\mu}\left(\max_{x\in\mathcal{X}} |f_t(x) - f_{t-1}(x)| + \hlambda\delta\right)}
    \end{equation*}
\end{lemma}

\begin{corollary} \label{corollary:bounded_distance_between_lambda_tlambda}
    For any $t\in\{2,3,..,T\}$,
    the distance between consecutive dual optimal values $\tlambda_t^*$ and $\tlambda_{t-1}^*$ corresponding to $\td_t$ and $\td_{t-1}$, respectively, is bounded as follows:
    $ |\tlambda_t^* - \tlambda_{t-1}^*| \leq \frac{2\hdelta_t}{\mu_d}$.
\end{corollary}

\subsubsection{Dual Regret Analysis}
Now, we bound the danger-aware dual regret of Alg.~\ref{alg:safe_dual_alg_warm_start} defined as: $R_{\td}(T) = \sum_{t=1}^{T} \td_t(\tlambda_t^*) - \td_t(\lambda_t)$ (Eq.~(\ref{eq:dual_reg_def})),
where $\tlambda_t^* = \arg\max_{\lambda\geq 0} \td_t(\lambda)$. This enables bounding the primal regret using Lemma~\ref{lemma:primal_dual_regret_relation}.
Note that the safety criterion in Theorem~\ref{theorem:safe_stepsize_bounds} implies two different behaviors of Alg.~\ref{alg:safe_dual_alg_warm_start} with different step sizes. Accordingly, we analyze each case separately. In general, a complete run from $t=1,...,T$ will consist of $n$ safe phases (in which $\nabla \td_{t-1}(\lambda_{t-1})\leq0$ for any $t$ during any safe phase) interleaved with $m$ danger phases (in which $\nabla \td_{t-1}(\lambda_{t-1})>0$ for any $t$ during any danger phase).
We denote the length of the $i$'th safe phase and the $j$'th danger phase by $\T^{S}_i$ and $\T^{D}_j$, respectively. Note that by definition, $\sum_{i=1}^n \T^{S}_i + \sum_{j=1}^m \T^{D}_j = T$.
Now, we bound the regret.

\begin{theorem}\label{theorem:regret_analyis_our_method}
    Consider a safe online problem (\ref{eq:orignal_online_problem}) with horizon $T$. Under Assumptions \ref{assum:bounded_set}-\ref{assum:safe_starting_point},
    Alg.~\ref{alg:safe_dual_alg_warm_start} with $\gamma_t = \mu/L_g^2$ when $\nabla\tilde{d}_{t-1}(\lambda_{t-1}) \leq 0$ and $\gamma_t = 2/\mu_d$ when $\nabla\tilde{d}_{t-1}(\lambda_{t-1}) > 0$, where $t\in[T]$,
    guarantees safety and the following sublinear primal regret w.r.t the comparator sequence defined in Eq.~(\ref{eq:primal_baseline}):
    \begin{equation*}
        \Reg_f(T) = \Ord\left(\sqrt{(V_{f,T}+V_{g,T})T}\right)
    \end{equation*}
\end{theorem}
\begin{proofsketch}
    We analyze and bound the dual regret in each phase separately, then we use these results to bound the primal regret using Lemma \ref{lemma:primal_dual_regret_relation} (see the full proof in Appendix \ref{appendix:full_regret_analysis_our_method}).

    \paragraph{The Danger Phase.} 
    We analyze the total dual regret over all $m$ danger phases. We do so by first bounding the single-step regret incurred at some step $t$ during any danger phase, defined as:
    \begin{equation*}
        r_{\td,t} = \td_t(\tlambda_t^*) - \td_t(\lambda_t).
    \end{equation*}
    Note that, by definition of the "danger phase", $\nabla \td_{t-1}(\lambda_{t-1}) > 0$, and thus by Theorem~\ref{theorem:safe_stepsize_bounds}, Alg.~\ref{alg:safe_dual_alg_warm_start} with $\gamma_t = 2/\mu_d$ ensures safety .
    Additionally, note that by Lemma \ref{lemma:slowly_changing_dual_gradients}, for any step $t$ we have $\nabla \td_t(\lambda_t) \leq \nabla \td_{t-1}(\lambda_t) + \hdelta_t \leq \hdelta_t$,
    where the second inequality follows from Lemma \ref{lemma:safety_in_terms_of_dual} since Alg.~\ref{alg:safe_dual_alg_warm_start} ensures safety by Theorem \ref{theorem:safe_stepsize_bounds}.
    Now, we bound the single-step regret:
    \begin{align}\label{eq:single_step_dual_reg_bound}
        r_{\td,t} &= \td_t(\tlambda_t^*) - \td_t(\lambda_t) \overset{(1)}{\leq} \langle \nabla\td_t(\lambda_t), \tlambda_t^* - \lambda_t \rangle \\
        &\leq |\nabla \td_t(\lambda_t)| \cdot |\tlambda_t^* - \lambda_t| \overset{(2)}{\leq} \hdelta_t |\tlambda_t^* - \lambda_t|,
    \end{align}
    where (1) follows by the concavity of $\td_t(\lambda)$ 
    and (2) follows since $0<\nabla \td_t(\lambda_t)\leq\hdelta_t$.
    
    Now, before bounding $|\tlambda_t^* - \lambda_t|$, note that $\nabla \td_{t-1}(\lambda_{t-1}) > 0$, by definition of the "danger phase", implies $\lambda_{t-1} \leq \tlambda_{t-1}^*$ since $\nabla \td_{t-1}(\lambda)$ is monotonically non-increasing and $\nabla \td_{t-1}(\tlambda_{t-1}^*) = 0$. Moreover, the safety criterion in Lemma \ref{lemma:safety_in_terms_of_dual} implies that $\nabla \td_{t-2}(\lambda_{t-1}) \leq 0$ which similarly implies $\lambda_{t-1} \geq \tlambda_{t-2}^*$. Thus, in total we have $\tlambda_{t-2}^* \leq \lambda_{t-1} \leq \tlambda_{t-1}^*$.
    Now, we bound $|\tlambda_t^* - \lambda_t|$:
    \begin{align}
        |\tlambda_t^* - \lambda_t| &\overset{(1)}{=} \left|\tlambda_t^* - (\lambda_{t-1} + \frac{2}{\mu_d} \nabla\td_{t-1}(\lambda_{t-1}))\right| \nonumber\\
        &\overset{(2)}{\leq} |\tlambda_t^* - \tlambda_{t-1}^*| + |\tlambda_{t-1}^* - \lambda_{t-1}| + \frac{2}{\mu_d} \hdelta_{t-1} \nonumber
        \\
        &\overset{(3)}{\leq} |\tlambda_t^* - \tlambda_{t-1}^*| + |\tlambda_{t-1}^* - \tlambda_{t-2}^*| + \frac{2}{\mu_d} \hdelta_{t-1} \nonumber\\
        &\overset{(4)}{\leq} \frac{2}{\mu_d}\hdelta_t + \frac{4}{\mu_d}\hdelta_{t-1}, \label{eq:case_2_iterate_bound_pt2}
    \end{align}
    where (1) follows by Alg.~\ref{alg:safe_dual_alg_warm_start}, (2) by the triangle inequality and since $0<\nabla \td_{t-1}(\lambda_{t-1})\leq\hdelta_{t-1}$, (3) since $\tlambda_{t-2}^* \leq \lambda_{t-1} \leq \tlambda_{t-1}^*$, and (4) by Corollary \ref{corollary:bounded_distance_between_lambda_tlambda}.

    Now, to analyze the total danger phase dual regret, we first set a new counter for each danger phase $j$, denoted by $\tau = 1,2,...,\T^D_j$. Note that this counter resets after every phase.
    Thus, the total dual regret incurred over all $m$ danger phases, which we denote by $\Reg_{\td}^D$, is bounded as follows:
    \begin{align*}
        \Reg_{\td}^{D} &= \sum_{j=1}^{m}\sum_{\tau=1}^{\T^{D}_j} r_{\td,\tau} \overset{(1)}{\leq} \sum_{j=1}^{m}\sum_{\tau=1}^{\T^{D}_j} \hdelta_\tau|\tlambda_{\tau}^* - \lambda_{\tau}|\\ 
        &\overset{(2)}{\leq} \frac{2}{\mu_d}\sum_{j=1}^{m}\sum_{\tau=1}^{\T^{D}_j} \hdelta_\tau^2 + 2\hdelta_\tau\hdelta_{\tau-1}
        \overset{(3)}{\leq} \frac{2}{\mu_d}\sum_{j=1}^{m}\sum_{\tau=1}^{\T_j^{D}}2\hdelta_\tau^2 + \hdelta_{\tau-1}^2,
    \end{align*}
    where (1) follows by Eq.~(\ref{eq:single_step_dual_reg_bound}), (2) by Eq.~(\ref{eq:case_2_iterate_bound_pt2}), and (3) since $\forall a,b\in\R: 2ab \leq a^2 + b^2$. Plugging in the expression for $\hdelta_{\tau}$ (Lemma \ref{lemma:slowly_changing_dual_gradients}):
    \begin{equation*}
        \hdelta_{\tau} = \delta + L_g\sqrt{\frac{2}{\mu}\left(\max_{x\in\mathcal{X}} |f_{\tau}(x) - f_{\tau-1}(x)| + \hlambda\delta\right)},
    \end{equation*}
    then using the fact that $\sqrt{X+Y} \leq \sqrt{X}+\sqrt{Y}$, $\forall X,Y \geq 0$, applying Jensen's inequality, and noting that $\sum_{j=1}^{m}\sum_{\tau=1}^{\T_j^D}1 \leq T$, $V_{f,T} = o(T)$, and $\delta=o(T^{-\alpha})$ (since $V_{g,T}=o(T)$), we have (see App.~ \ref{appendix:full_regret_analysis_our_method}):
    \begin{align*}
        \Reg_{\td}^D &\leq \frac{6}{\mu_d}\Bigg(\delta^2 T + 2L_g\sqrt{\frac{2}{\mu}}\delta\sqrt{TV_{f,T}} + 2L_g\sqrt{\frac{2\hlambda}{\mu}}\delta^{\frac{3}{2}}T +\\
        &\quad+ L^2_g\frac{2}{\mu}V_{f,T} + L^2_g\frac{2}{\mu}\hlambda\delta T \Bigg)
        = \Ord(\delta T + V_{f,T}).
    \end{align*}

    \paragraph{The Safe Phase.} We analyze the total regret incurred over all $n$ safe phases, where each safe phase $i$ lasts for $\T^S_i$ steps.
    We set a new counter for the steps during each safe phase $i$, denoted by $\tau = 1,2,...,\T^S_i$, which resets after every phase. Note that throughout any safe phase $i$, $\forall \tau\in[\T^S_i]$, $\nabla \td_{\tau-1}(\lambda_{\tau-1}) \leq 0$, and thus Alg.~\ref{alg:safe_dual_alg_warm_start} with $\gamma_t = \mu/L^2_g$ ensures safety by Theorem \ref{theorem:safe_stepsize_bounds}.
    To analyze the dual regret, we use the following lemma which provides two helpful properties (see Appendix \ref{appendix:proof_lemma_helpful_properties} for the proof). Throughout this analysis we denote $z_\tau = -\nabla \td_{\tau}(\lambda_{\tau})$.
    \begin{lemma}\label{lemma:helpful_properties_for_case1}
        For any $i\in[n]$, Alg.~\ref{alg:safe_dual_alg_warm_start} ensures:
        (A) $\sum_{\tau=1}^{\T^S_i} z_{\tau} \leq \frac{\hlambda L_g^2}{\mu}$,$\quad$ (B) $\sum_{\tau=1}^{\T^S_i} z_{\tau} \leq \frac{6L_g^2}{\mu\mu_d} \sum_{\tau=1}^{\T^S_i+1}\hdelta_{\tau}$.
    \end{lemma}

    \noindent Now, we bound the total dual regret incurred over all $n$ safe phases, which we denote by $\Reg_{\td}^S$:
    \begin{align*}
        \Reg_{\td}^S &= \sum_{i=1}^{n}\sum_{{\tau}=1}^{\T^S_i} \td_{\tau}(\tlambda_{\tau}^*) - \td_{\tau}(\lambda_{\tau})\\ 
        &\overset{(1)}{\leq} \sum_{i=1}^{n}\sum_{{\tau}=1}^{\T^S_i} \langle -\nabla \td_{\tau}(\lambda_{\tau}), \lambda_{\tau} - \tlambda_{\tau}^* \rangle -\frac{\mu_d}{2} |\lambda_{\tau} - \tlambda_{\tau}^*|^2
        \\
        &= \sum_{i=1}^{n}\sum_{{\tau}=1}^{\T^S_i}\Bigg( -\frac{1}{2}\left| \sqrt{\mu_d}(\tlambda_{\tau}^* - \lambda_{\tau}) - \frac{1}{\sqrt{\mu_d}}\nabla \td_{\tau}(\lambda_{\tau}) \right|^2 \\
        &\quad\quad\quad\quad\quad\;\,+ \frac{1}{2\mu_d}|\nabla \td_{\tau}(\lambda_{\tau})|^2\Bigg) \\
        &\leq \frac{1}{2\mu_d}\sum_{i=1}^{n}\sum_{{\tau}=1}^{\T^S_i} z_{\tau}^2 
        \overset{(2)}{\leq} \frac{1}{2\mu_d}
        \sum_{i=1}^{n}\left(\sum_{{\tau}=1}^{\T^S_i} z_{\tau} \right)^2 \\
        &\overset{(3)}{\leq} \frac{1}{2\mu_d}\frac{\hlambda L^2_g}{\mu} \sum_{i=1}^{n}\sum_{{\tau}=1}^{\T^S_i} z_{\tau}
        \overset{(4)}{\leq} 3\hlambda\left(\frac{L^2_g}{\mu\mu_d}\right)^2 \sum_{i=1}^{n}\sum_{\tau=1}^{\T^S_i+1}\hdelta_{\tau}
    \end{align*}
    where (1) follows by Lemma \ref{lemma:local_str_conc_of_dual}, (2) since $z_{\tau} \geq 0, \forall \tau\in[\T^S_i]$, and (3) and (4) by properties (A) and (B) in Lemma \ref{lemma:helpful_properties_for_case1}, respectively.
    Following similar steps as in the previous analysis, and denoting $\beta = 3\hlambda (L_g^2/\mu\mu_d)^2$:
    \begin{align*}
        \Reg_{\td}^S &\leq \beta \left(\delta T + \sqrt{\frac{2}{\mu}}L_g \sqrt{TV_{f,T}} + \sqrt{\frac{2\hlambda}{\mu}}L_g \sqrt{\delta}T \right) \\
        &= \Ord\left(\sqrt{TV_{f,T}} + \sqrt{\delta}T\right)
    \end{align*}
    \paragraph{Putting It All Together.} Combining the dual regret of all danger and safe phases, the total dual regret is bounded as follows:
    \begin{align*}
        \Reg_{\td}(T) = \Reg^S_{\td} + \Reg^D_{\td} \overset{(1)}{=} \Ord\left(\sqrt{(V_{f,T}+V_{g,T})T}\right),
    \end{align*}
    where (1) follows since $V_{f,T}=o(T)$, $V_{g,T}=\delta T$, and $\delta = o(T^{-\alpha})$, with $\alpha>0$. Now, recall that by Lemma \ref{lemma:primal_dual_regret_relation}, $\Reg_f(T) = \Ord\left(\max\left\{ \hat{\Reg}_{\td}(T), \sqrt{(V_{f,T}+V_{g,T})T}\right\}\right)$.
    Thus, by plugging in the bound on $\Reg_{\td}(T)$, we obtain the following bound on the primal regret:
    \begin{align*}
        R_f(T) = \Ord\left(\sqrt{(V_{f,T}+V_{g,T})T}\right).
    \end{align*}
\end{proofsketch}

\section{EXTENSION TO THE CONVEX CASE}
We extend our results to the convex case, where the loss functions are convex but \emph{not} necessarily strongly convex.
We use the following approach, inspired by \cite{Allen-Zhu16_convex_trick}. Let $\{\hf_t\}_{t=1}^T$, where $\hf_t:\R^D\rightarrow \R, \forall t\in[T]$, be convex but \emph{not necessarily} strongly convex functions. We define the following surrogate functions:
\begin{equation}
    f_t(x) = \hf_t(x) + \frac{\mu}{2}\|x\|^2, \forall t\in[T].
\end{equation}
where $\mu>0$. Note that, by definition, $f_t$ is $\mu$-strongly convex, $\forall t\in[T]$.

Note that Theorem \ref{theorem:guarantees_naive_alg} and Theorem \ref{theorem:regret_analyis_our_method} provide bounds on the regret in terms of $\{f_t\}_{t=1}^T$.
Furthermore, we show that the regret in terms of $\{\hf_t\}_{t=1}^T$ can be related to the regret in terms of $\{f_t\}_{t=1}^T$. Thus, by leveraging the existing bounds and optimizing over $\mu$, we obtain $\Ord\left(\left(V_{f,T}+V_{g,T}\right)^\frac{1}{3}T^\frac{2}{3}\right)$ and $\Ord\left(\left(V_{f,T}+V_{g,T}\right)^\frac{1}{7}T^\frac{6}{7}\right)$ regret (in terms of $\{\hf\}_{t=1}^T$) for  Alg.~\ref{alg:naive} and Alg.~\ref{alg:safe_dual_alg_warm_start}, respectively. Please refer to Appendix \ref{appendix:convex_case} for the full analysis and proof.
\section{CONCLUSION}\label{sec:conclusion}
We presented the first theoretical guarantees for safe online learning problems with dynamically evolving constraints, which are more applicable to real-world scenarios.
Our results address a significant gap in research on constrained OCO and demonstrate that safety, via zero constraint violation, and sublinear regret can be achieved simultaneously.
This is accomplished through a novel dual approach by transforming the primal safety criterion to the dual space and employing OGA with a dichotomous learning rate.
Furthermore, we established an intriguing relationship between the primal regret and the dual regret and leveraged it to bound the primal regret.
Our work is the first to guarantee absolute safety, in the form of zero constraint violation, and sublinear primal regret.
An interesting direction for future research is to explore lower bounds for this setting.
\section*{Acknowledgements}
This research was partially supported by Israel PBC-VATAT, by the Technion Artificial Intelligent Hub (Tech.AI), and by the Israel Science Foundation (grant No. 3109/24).
The first author would like to thank VATAT (through the Israel Council for Higher Education) for supporting this research.

\bibliography{references}
\section*{Checklist}



 \begin{enumerate}

 \item For all models and algorithms presented, check if you include:
 \begin{enumerate}
   \item A clear description of the mathematical setting, assumptions, algorithm, and/or model. [Yes]
   \item An analysis of the properties and complexity (time, space, sample size) of any algorithm. [Not Applicable]
   \item (Optional) Anonymized source code, with specification of all dependencies, including external libraries. [Not Applicable]
 \end{enumerate}

 \item For any theoretical claim, check if you include:
 \begin{enumerate}
   \item Statements of the full set of assumptions of all theoretical results. [Yes]
   \item Complete proofs of all theoretical results. [Yes]
   \item Clear explanations of any assumptions. [Yes]     
 \end{enumerate}

 \item For all figures and tables that present empirical results, check if you include:
 \begin{enumerate}
   \item The code, data, and instructions needed to reproduce the main experimental results (either in the supplemental material or as a URL). [Not Applicable]
   \item All the training details (e.g., data splits, hyperparameters, how they were chosen). [Not Applicable]
   \item A clear definition of the specific measure or statistics and error bars (e.g., with respect to the random seed after running experiments multiple times). [Not Applicable]
   \item A description of the computing infrastructure used. (e.g., type of GPUs, internal cluster, or cloud provider). [Not Applicable]
 \end{enumerate}

 \item If you are using existing assets (e.g., code, data, models) or curating/releasing new assets, check if you include:
 \begin{enumerate}
   \item Citations of the creator If your work uses existing assets. [Not Applicable]
   \item The license information of the assets, if applicable. [Not Applicable]
   \item New assets either in the supplemental material or as a URL, if applicable. [Not Applicable]
   \item Information about consent from data providers/curators. [Not Applicable]
   \item Discussion of sensible content if applicable, e.g., personally identifiable information or offensive content. [Not Applicable]
 \end{enumerate}

 \item If you used crowdsourcing or conducted research with human subjects, check if you include:
 \begin{enumerate}
   \item The full text of instructions given to participants and screenshots. [Not Applicable]
   \item Descriptions of potential participant risks, with links to Institutional Review Board (IRB) approvals if applicable. [Not Applicable]
   \item The estimated hourly wage paid to participants and the total amount spent on participant compensation. [Not Applicable]
 \end{enumerate}

 \end{enumerate}

\onecolumn
\appendix
\newpage
\section{Proofs of Section 2}\label{appendix:proofs_sec_2}
We first show a useful Lemma on the strong convexity and smoothness of the Lagrangians.
\subsection{Strong Convexity and Smoothness of the Lagrangian}\label{appendix:smooth_str_conv_lagrangian}
\begin{lemma}\label{lemma:smooth_str_conv_lagrangian}
    Under Assumptions \ref{assum:obj_smooth_strconv_lipsc}-\ref{assum:constr_conv_lipsc}, the Lagrangian $\Lag_t(x,\lambda) = f_t(x) + \lambda g_t(x)$ and the Lagrangian $\TLag_t(x,\lambda) = f_t(x) + \lambda (g_t(x) + \delta)$ are $\mu$-strongly convex and $M$-smooth in $x$, with $M = M_f + \lambda M_g$, $\forall t\in[T], \forall \lambda \geq 0$.
\end{lemma}
\begin{proof}
    We start with strong convexity. Since, $\forall t\in[T]$, $f_t$ is $\mu$-strongly convex and $g_t$ is convex, we have that for any $t\in[T]$, any $\lambda \geq 0$, and any $x,y\in\mathcal{X}$:
\begin{align}
    \Lag_t(y,\lambda) &= f_t(y) + \lambda g_t(y) \\
    &\geq f_t(x) + \langle \nabla f_t(x), y-x \rangle + \frac{\mu}{2}\norm{y-x}^2 + \lambda\left(g_t(x) + \langle \nabla g_t(x), y-x \rangle\right) \\
    &=f_t(x) + \lambda g_t(x) + \langle \nabla f_t(x) + \lambda \nabla g_t(x), y-x \rangle + \frac{\mu}{2}\norm{y-x}^2 \\
    &= \Lag_t(x,\lambda) + \langle \nabla_x \Lag_t(x,\lambda), y-x \rangle + \frac{\mu}{2}\norm{y-x}^2
\end{align}

Namely, $\Lag_t(x,\lambda)$ is $\mu$-strongly convex in $x$, $\forall t\in[T], \forall \lambda \geq 0$. Substituting $g_t(x) \leftarrow g_t(x) + \delta$ shows that $\TLag_t(x,\lambda)$, too, is $\mu$-strongly convex $\forall t\in[T], \forall \lambda \geq 0$ since $g_t(x) + \delta$ is convex as well.

\noindent We now prove the smoothness of the Lagrangian. For any $t\in[T]$, any $\lambda \geq 0$, and any $x,y\in\mathcal{X}$:
\begin{align}
    \Lag_t(y,\lambda) &= f_t(y) + \lambda g_t(y) \\
    & \leq f_t(x) + \langle \nabla f_t(x), y-x \rangle + \frac{M_f}{2}\norm{y-x}^2 +\\
    &\quad\quad + \lambda \left(g_t(x) + \langle \nabla g_t(x), y-x \rangle + \frac{M_g}{2}\norm{y-x}^2\right) \\
    &= f_t(x) + \lambda g_t(x) + \langle \nabla f_t(x) + \lambda \nabla g_t(x), y-x \rangle + \frac{M_f + \lambda M_g}{2}\norm{y-x}^2 \\
    &= \Lag_t(x) + \langle \nabla_x \Lag_t(x,\lambda), y-x \rangle + \frac{M}{2} \norm{y-x}^2
\end{align}
Namely, $\Lag_t(x,\lambda)$ is $M$-smooth in $x$, $\forall t\in[T], \forall \lambda \geq 0$. Substituting $g_t(x) \leftarrow g_t(x) + \delta$ shows that $\TLag_t(x,\lambda)$, too, is $M$-smooth $\forall t\in[T], \forall \lambda \geq 0$ since $g_t(x) + \delta$ is an $M_g$-smooth function as well.
\end{proof}

\subsection{Proof of Lemma \ref{lemma:universal_bound_dual_baseline}: A Universal Bound on the Optimal Dual Values}\label{appendix:universal_bound_lambda}
\begin{proof}
    For any point $x_t^0\in\mathcal{X}_t$ such that $g_t(x_t^0) < 0$ (such $x_t^0$ necessarily exists by Assumption \ref{assum:non_shallow_constr_and_strong_duality}), and by the optimality of $x_t^*$ and $\lambda_t^*$, we have:
    \begin{equation}
        \Lag_t(x_t^0,\lambda_t^*) \geq \Lag_t(x_t^*,\lambda_t^*)
    \end{equation}
Decomposing the Lagrangian:
    \begin{align}
        f_t(x_t^0) + \lambda_t^* g_t(x_t^0) &\geq f_t(x_t^*) + \lambda_t^* g_t(x_t^*) \\
        f_t(x_t^0) + \lambda_t^* g_t(x_t^0) &\geq f_t(x_t^*)
    \end{align}
where the second line is due to complementary slackness. Rearranging:
    \begin{align}
        \lambda_t^* &\leq \frac{f_t(x_t^0) - f_t(x_t^*)}{-g_t(x_t^0)}\leq \frac{L_f\norm{x_t^0 - x_t^*}}{-g_t(x_t^0)}\leq \frac{L_fR}{-g_t(x_t^0)}
    \end{align}
where the second inequality is by the $L_f$-Lipschitz continuity of the loss functions, and the last inequality is by Assumption \ref{assum:bounded_set} (bounded set).
Note that this bound holds for any $x_t^0$ such that $g_t(x_t^0) < 0$.
Furthermore, Assumption \ref{assum:non_shallow_constr_and_strong_duality} implies $g_t(x_t^0) \leq -G$. Thus, since $g_t(x)$ is continuous $\forall t\in[T]$, there exists some point $x'_t$ such that $g_t(x'_t) = -G$, thus:
\begin{equation}
    \lambda_t^* \leq \frac{L_fR}{-g_t(x'_t)} = \frac{L_fR}{G}
\end{equation}
Namely, $\forall t\in[T]: \lambda_t^* \leq \hat{\lambda}$, with $\hat{\lambda} = \frac{L_f R}{G}$.
Substituting $g_t(x) \leftarrow g_t(x) + \delta$ and following the same proof yields the bound on $\tlambda_t^*$.
\end{proof}

\subsection{Proof of Lemma \ref{lemma:local_str_conc_of_dual}: Local Strong Concavity of the Dual Function}\label{appendix:value_of_mu_d}
\begin{proof}
    For any $t\in[T]$, the Hessian of the dual function $d_t(\lambda)$, corresponding to the optimization problem $\min_{x\in\mathcal{X}} f_t(x) \;\text{s.t.}\; g_t(x)\leq0$, is given by (Eq. (6.9), page 598 in \citet{bertsekas1999}):
\begin{align}\label{eq:dual_hessian}
    \nabla^2_{\lambda}d_t(\lambda) & = -\nabla_x g_t(x_{t,\lambda}^*)^T(\nabla_x^2 f_t(x_{t,\lambda}^*) + \lambda \nabla_x^2 g_t(x_{t,\lambda}^*))^{-1}\nabla_x g_t(x_{t,\lambda}^*).
\end{align}
Note that in our case of a single constraint, $d_t(\lambda)$ is a one-dimensional function, and the Hessian is simply a scalar. Since, $\forall t\in[T]$, $f_t$ is $\mu$-strongly convex and $M_f$-smooth and $g_t$ is convex and $M_g$-smooth, we have:
\begin{equation}
    \mu \preceq \nabla_x^2 f_t(x_{t,\lambda}^*) + \lambda \nabla_x^2 g_t(x_{t,\lambda}^*) \preceq M_f + \lambda M_g.
\end{equation}
Thus, by Eq.~\eqref{eq:dual_hessian}:
\begin{equation}
    \nabla^2_{\lambda}d_t(\lambda) \preceq -\frac{1}{M_f + \lambda M_g}\|\nabla_x g_t(x_{t,\lambda}^*)\|^2.
\end{equation}
We now lower bound $\|\nabla_x g_t(x_{t,\lambda}^*)\|$ on the set $\{\lambda \geq 0 : g_t(x_{t,\lambda}^*) \geq -G/2\}$. By Assumption \ref{assum:non_shallow_constr_and_strong_duality}, there exists $x_t^0\in\mathcal{X}$ such that $g_t(x_t^0) \leq -G$. Thus: 
\begin{align}
    \|\nabla_x g_t(x_{t,\lambda}^*)\| \geq \left\la \nabla_x g_t(x_{t,\lambda}^*), \frac{x_{t,\lambda}^* - x_t^0}{\|x_{t,\lambda}^* - x_t^0\|}\right\ra \geq \frac{g_t(x_{t,\lambda}^*) - g_t(x_t^0)}{\|x_{t,\lambda}^* - x_t^0\|} \geq \frac{g_t(x_{t,\lambda}^*) - g_t(x_t^0)}{R} \geq \frac{G}{2R}
\end{align}
where the first inequality is due to the Cauchy-Schwartz inequality, the second follows since $g_t$ is convex (Assumption \ref{assum:constr_conv_lipsc}), and the third follows from Assumption \ref{assum:bounded_set} (bounded set).
Therefore, the dual function $d_t$ is locally $\mu_d$-strongly concave on the set $\{\lambda \geq 0 : g_t(x_{t,\lambda}^*) \geq -G/2\}$ with:
\begin{equation}
    \mu_d = \frac{G^2}{4R^2(M_f + \lambda M_g)}
\end{equation}
Plugging in $g_t(x) \leftarrow g_t(x) + \delta$ and following the same analysis shows that $\td_t$ is locally $\mu_d$-strongly concave as well.
\end{proof}

\newpage
\section{Proofs for Alg.~\ref{alg:naive}}\label{appendix:proofs_sec_3}
\subsection{Bounding the distance between the iterates and the comparator in Alg.~\ref{alg:naive}}\label{appendix:proof_lemma_dist_iter_baseline_bound}

\begin{lemma}\label{lemma:bound_distance_iter_baseline}
    Under Assumptions \ref{assum:bounded_set}-\ref{assum:slowly_changing_constr}, \ref{assum:non_shallow_constr_and_strong_duality}, the distance between the iterates $x_t$ of Alg.~\ref{alg:naive}, and their corresponding comparators $x_t^*$, defined in Eq.~(\ref{eq:primal_baseline}), is bounded as follows, $\forall t\in\{2,3,...,T\}$:
    \begin{equation*}
        \norm{x_t - x_t^*} \leq \sqrt{\frac{2}{\mu}\left(\max_{x\in\mathcal{X}} |f_t(x) - f_{t-1}(x)| + \hat{\lambda}\delta\right)}.
    \end{equation*}
See Appendix \ref{appendix:proof_lemma_dist_iter_baseline_bound} for the proof.
\end{lemma}

\begin{proof}
    Note that $x_t$ in Alg.~\ref{alg:naive} and the comparator defined in Eq.~(\ref{eq:primal_baseline}) can be equivalently written using the Lagrangian formulation, as follows:
\begin{align}
    &x_t = \arg\min_{x\in\mathcal{X}} \max_{\lambda \geq 0} f_{t-1}(x) + \lambda(g_{t-1}(x)+\delta) \triangleq \arg\min_{x\in\mathcal{X}} \tilde{\Lag}_{t-1}(x,\tlambda_{t-1}^*)\label{eq:optimality_iter} \\
    &x_t^* = \arg\min_{x\in\mathcal{X}} \max_{\lambda \geq 0} f_t(x) + \lambda g_t(x) \triangleq \arg\min_{x\in\mathcal{X}} \Lag_t(x,\lambda_t^*)\label{eq:optimality_baseline},
\end{align}
where we define, since under Assumption \ref{assum:non_shallow_constr_and_strong_duality} strong duality holds:
\begin{align}
    &\tlambda_{t-1}^* = \arg\max_{\lambda \geq 0} \min_{x\in\mathcal{X}} f_{t-1}(x) + \lambda(g_{t-1}(x) + \delta) = \arg\max_{\lambda \geq 0} \min_{x\in\mathcal{X}} \TLag_{t-1}(x,\lambda) \\
    &\lambda_t^* = \arg\max_{\lambda \geq 0} \min_{x\in\mathcal{X}} f_t(x) + \lambda g_t(x) = \arg\max_{\lambda \geq 0} \min_{x\in\mathcal{X}} \Lag_t(x,\lambda).
\end{align}
Thus, by Lemma~\ref{lemma:smooth_str_conv_lagrangian} on the strong convexity of the Lagrangians:
\begin{align}
    \TLag_{t-1}(x_t^*,\tlambda_{t-1}^*) &\geq \TLag_{t-1}(x_t,\tlambda_{t-1}^*) + \langle \nabla_x\TLag_{t-1}(x_t,\tlambda_{t-1}^*), x_t^* - x_t \rangle + \frac{\mu}{2}\norm{x_t - x_t^*}^2 \\
    \Lag_t(x_t,\lambda_t^*) &\geq \Lag_t(x_t^*,\lambda_t^*) + \langle \nabla_x\Lag_t(x_t^*,\lambda_t^*), x_t - x_t^* \rangle + \frac{\mu}{2}\norm{x_t - x_t^*}^2,
\end{align}
and by Eq. (\ref{eq:optimality_iter})-(\ref{eq:optimality_baseline}) on the optimality of $x_t$ and $x_t^*$:
\begin{align}
    \TLag_{t-1}(x_t^*,\tlambda_{t-1}^*) &\geq \TLag_{t-1}(x_t,\tlambda_{t-1}^*) + \frac{\mu}{2}\norm{x_t - x_t^*}^2 \\
    \Lag_t(x_t,\lambda_t^*) &\geq \Lag_t(x_t^*,\lambda_t^*) + \frac{\mu}{2}\norm{x_t - x_t^*}^2.
\end{align}
Now, decomposing the Lagrangians yields:
\begin{align}
    f_{t-1}(x_t^*) + \tlambda_{t-1}^* (g_{t-1}(x_t^*) + \delta) &\geq f_{t-1}(x_t) + \tlambda_{t-1}^* (g_{t-1}(x_t) + \delta) +  \frac{\mu}{2}\norm{x_t - x_t^*}^2 \\
    f_t(x_t) + \lambda_t^* g_t(x_t) &\geq f_t(x_t^*) + \lambda_t^* g_t(x_t^*) +  \frac{\mu}{2}\norm{x_t - x_t^*}^2,
\end{align}
and note that by complementary slackness $\lambda_t^* g_t(x_t^*) = 0$ and $\tlambda_{t-1}^* (g_{t-1}(x_t)+\delta) = 0$, thus:
\begin{align}
    f_{t-1}(x_t^*) + \tlambda_{t-1}^* (g_{t-1}(x_t^*) + \delta) &\geq f_{t-1}(x_t) +  \frac{\mu}{2}\norm{x_t - x_t^*}^2 \\
    f_t(x_t) + \lambda_t^* g_t(x_t) &\geq f_t(x_t^*) + \frac{\mu}{2}\norm{x_t - x_t^*}^2.
\end{align}
Finally, by summing the two equations and rearranging, we have:
\begin{align}
    \mu\norm{x_t - x_t^*}^2 &\leq f_{t-1}(x_t^*) - f_t(x_t^*) + f_t(x_t) - f_{t-1}(x_t) + \tlambda_{t-1}^* (g_{t-1}(x_t^*)+\delta) + \lambda_t^*g_t(x_t) \\
    &\leq f_{t-1}(x_t^*) - f_t(x_t^*) + f_t(x_t) - f_{t-1}(x_t) + \tlambda_{t-1}^* (g_{t-1}(x_t^*)+\delta) \\
    &\leq f_{t-1}(x_t^*) - f_t(x_t^*) + f_t(x_t) - f_{t-1}(x_t) + 2\tlambda_{t-1}^* \delta \\
    &\leq 2\max_{x\in\mathcal{X}} |f_t(x) - f_{t-1}(x)| + 2\hat{\lambda} \delta
\end{align}
where the second inequality is by Assumption \ref{assum:slowly_changing_constr} since $g_t(x_t) \leq g_{t-1}(x_t) + \delta \leq 0$; the third is because $g_{t-1}(x_t^*) \leq g_t(x_t^*) + \delta \leq \delta$; and the last is by Lemma~\ref{lemma:universal_bound_dual_baseline} on the universal bound of the optimal dual values. Dividing by $\mu$ and taking the square root concludes the proof.
\end{proof}

\subsection{Proof of Theorem \ref{theorem:guarantees_naive_alg}}\label{appendix:full_proof_naive_alg_guarantees}
\begin{proof}
    We first show that Alg.~\ref{alg:naive} is safe. 
    Note that:
    \begin{equation}
       g_t(x_t) \leq g_{t-1}(x_t) + \delta \leq 0,
    \end{equation}
    where the first inequality is by Assumption \ref{assum:slowly_changing_constr} (slowly changing constraints) and the second is by the update of $x_t$ in Alg.~\ref{alg:naive} as the solution of a constrained optimization problem with the constraint $g_{t-1}(x)+\delta \leq 0$. Additionally, by Assumption \ref{assum:safe_starting_point}, the first iterate $x_1$ of Alg.\ref{alg:naive} is safe.
    Therefore, the iterates $x_t$ of Alg.~\ref{alg:naive} satisfy the constraints in every step, namely $g_t(x_t) \leq 0, \forall t\in[T]$. Thus Alg.~\ref{alg:naive} guarantees zero constraint violation, i.e., safety.
    
    Now, we bound the regret in the strongly convex case, namely $f_t$ are $\mu$-strongly convex $\forall t\in[T]$:
    \begin{align}
        \Reg_f(T) &= \sum_{t=1}^{T} f_t(x_t) - f_t(x_t^*)\\ 
        &\overset{(1)}{\leq} L_f \norm{x_1 - x_1^*} + \sum_{t=2}^{T} L_f \norm{x_t - x_t^*} \\
        &\overset{(2)}{\leq} L_f R + \sum_{t=2}^{T} L_f \sqrt{\frac{2}{\mu}\left(\max_{x\in\mathcal{X}} |f_t(x) - f_{t-1}(x)| + \hat{\lambda}\delta\right)} \\
        &\overset{(3)}{\leq} L_f R + \sum_{t=2}^{T} \sqrt{\frac{2\hat{\lambda}}{\mu}}L_f \sqrt{\delta} + \sum_{t=2}^{T} \sqrt{\frac{2}{\mu}}L_f \sqrt{\max_{x\in\mathcal{X}} |f_t(x) - f_{t-1}(x)|} \\
        &\overset{(4)}{\leq} L_f R + \sqrt{\frac{2\hat{\lambda}}{\mu}}L_f \sqrt{\delta}T + \sqrt{\frac{2}{\mu}}L_f\sqrt{T\sum_{t=2}^T \max_{x\in\mathcal{X}} |f_t(x) - f_{t-1}(x)|} \\
        &\overset{(5)}{\leq} L_f R + \sqrt{\frac{2\hat{\lambda}}{\mu}}L_f \sqrt{V_{g,T}T} + \sqrt{\frac{2}{\mu}}L_f \sqrt{V_{f,T} T}
    \end{align}
    where (1) follows since $f_t$ is $L_f$-Lipschitz continuous $\forall t\in[T]$ (Assumption \ref{assum:obj_smooth_strconv_lipsc}); (2) follows from Assumption \ref{assum:bounded_set} and Lemma~\ref{lemma:bound_distance_iter_baseline}; (3) follows since $\forall X,Y\geq 0: \sqrt{X+Y} \leq \sqrt{X} + \sqrt{Y}$; (4) follows from Jensen's inequality; and (5) follows by Assumption \ref{assum:bounded_TV_obj} (bounded total variation of the loss functions $\{f_t(x)\}_{t=1}^T$) and since $V_{g,T} = \delta T$.
\end{proof}

\newpage
\section{Proofs of Our Main Approach}\label{appendix:proofs_sec_4}
\subsection{Proof of Corollary \ref{corollary:lipschitz_cont_of_dual}: Lipschitz Continuity of the Dual Gradients}\label{appendix:proof_corollary_lipsch_of_dual}
The proof of Corollary \ref{corollary:lipschitz_cont_of_dual} rests on the following helpful lemma.
\begin{lemma}\label{lemma:bound_distance_opt_x_diff_lambda}
    Under Assumptions \ref{assum:obj_smooth_strconv_lipsc}-\ref{assum:constr_conv_lipsc}, \ref{assum:non_shallow_constr_and_strong_duality}, the distance between $x_{t,\lambda_1}^* = \arg\min_{x\in\mathcal{X}}\TLag_t(x,\lambda_1)$ and $x_{t,\lambda_2}^* = \arg\min_{x\in\mathcal{X}}\TLag_t(x,\lambda_2)$, $\forall t\in[T], \forall \lambda_1, \lambda_2 \geq 0$, is bounded as follows:
     $\norm{x_{t,\lambda_1}^* - x_{t,\lambda_2}^*} \leq \frac{L_g}{\mu}|\lambda_1 - \lambda_2|.$
\end{lemma}
\begin{proof}
    Note that by definition of $x_{t,\lambda}^*$ in Eq.~(\ref{eq:x_t_lambda_opt}):
    \begin{align}
        &x_{t,\lambda_1}^* = \arg\min_{x\in\mathcal{X}}\TLag_t(x,\lambda_1)\label{eq:optimality_x_t_lambda_1} = \arg\min_{x\in\mathcal{X}}\Lag_t(x,\lambda_1)\\
        &x_{t,\lambda_2}^* = \arg\min_{x\in\mathcal{X}}\TLag_t(x,\lambda_2) = \arg\min_{x\in\mathcal{X}}\Lag_t(x,\lambda_2)\label{eq:optimality_x_t_lambda_2}.
    \end{align}
    By Lemma \ref{lemma:smooth_str_conv_lagrangian}, the Lagrangian $\Lag_t(x,\lambda)$ is strongly convex in $x$:
    \begin{align}
        &\Lag_{t}(x_{t,\lambda_2}^*,\lambda_1) \geq \Lag_{t}(x_{t,\lambda_1}^*,\lambda_1) + \langle \nabla_x \Lag_t(x_{t,\lambda_1}^*, \lambda_1), x_{t,\lambda_2}^* - x_{t,\lambda_1}^* \rangle +  \frac{\mu}{2}\norm{x_{t,\lambda_1}^* - x_{t,\lambda_2}^*}^2 \\
        &\Lag_{t}(x_{t,\lambda_1}^*,\lambda_2) \geq \Lag_{t}(x_{t,\lambda_2}^*,\lambda_2) + \langle \nabla_x \Lag_t(x_{t,\lambda_2}^*, \lambda_2), x_{t,\lambda_1}^* - x_{t,\lambda_2}^* \rangle + \frac{\mu}{2}\norm{x_{t,\lambda_1}^* - x_{t,\lambda_2}^*}^2
    \end{align}
    By Eq.~(\ref{eq:optimality_x_t_lambda_1})-(\ref{eq:optimality_x_t_lambda_2}) on the optimality of $x_{t,\lambda_1}^*$ and $ x_{t,\lambda_2}^*$, we have:
    \begin{align}
        &\Lag_{t}(x_{t,\lambda_2}^*,\lambda_1) \geq \Lag_{t}(x_{t,\lambda_1}^*,\lambda_1) + \frac{\mu}{2}\norm{x_{t,\lambda_1}^* - x_{t,\lambda_2}^*}^2 \\
        &\Lag_{t}(x_{t,\lambda_1}^*,\lambda_2) \geq \Lag_{t}(x_{t,\lambda_2}^*,\lambda_2) + \frac{\mu}{2}\norm{x_{t,\lambda_1}^* - x_{t,\lambda_2}^*}^2
    \end{align}
    Decomposing the Lagrangian,
    \begin{align}
        &f_t(x_{t,\lambda_2}^*) + \lambda_1 g_{t}(x_{t,\lambda_2}^*) \geq f_t(x_{t,\lambda_1}^*) + \lambda_1 g_{t}(x_{t,\lambda_1}^*) + \frac{\mu}{2}\norm{x_{t,\lambda_1}^* - x_{t,\lambda_2}^*}^2 \\
        &f_t(x_{t,\lambda_1}^*) + \lambda_2 g_{t}(x_{t,\lambda_1}^*) \geq f_t(x_{t,\lambda_2}^*) + \lambda_2 g_{t}(x_{t,\lambda_2}^*) + \frac{\mu}{2}\norm{x_{t,\lambda_1}^* - x_{t,\lambda_2}^*}^2
    \end{align}
    Summing, rearranging, and using the Lipschitz continuity of the constraints,
    \begin{align}
        \mu\norm{x_{t,\lambda_1}^* - x_{t,\lambda_2}^*}^2 &\leq (\lambda_1 - \lambda_2)(g_{t}(x_{t,\lambda_2}^*) - g_{t}(x_{t,\lambda_1}^*)) \\
        &\leq |\lambda_1 - \lambda_2| L_g \norm{x_{t,\lambda_1}^* - x_{t,\lambda_2}^*}
    \end{align}
    Thus, we have:
    \begin{equation}
        \norm{x_{t,\lambda_1}^* - x_{t,\lambda_2}^*} \leq \frac{L_g}{\mu}|\lambda_1 - \lambda_2|
    \end{equation}
\end{proof}

Now we prove Corollary \ref{corollary:lipschitz_cont_of_dual}.
\begin{proof}
    By definition of the gradient of the dual function. we have:
    \begin{align}
        |\nabla \td_t(\lambda_1) - \nabla \td_t(\lambda_2)|
        = & |g_t(x_{t,\lambda_1}^*) - g_t(x_{t,\lambda_2}^*)|\\
        \leq & L_g\norm{x^*_{t,\lambda_1} - x^*_{t,\lambda_2}} \\
        \leq & \frac{L_g^2}{\mu}|\lambda_1 - \lambda_2|
    \end{align}
    where the equality follows from the definition of the dual gradients, the first inequality follows from Assumption \ref{assum:constr_conv_lipsc} (Lipschitz continuity of the constraints), and the second inequality follows from Lemma~\ref{lemma:bound_distance_opt_x_diff_lambda}.
\end{proof}

\subsection{Full proof of Theorem \ref{theorem:safe_stepsize_bounds}}\label{appendix:full_proof_stepsize_bounds}
\begin{proof}
    For any $t\in[T]$, we split the proof according to the sign of $\nabla \td_{t-1}(\lambda_{t-1})$: \\
    \textbf{The Safe Phase:} $\nabla\tilde{d}_{t-1}(\lambda_{t-1}) \leq 0$. Note that:
    \begin{equation*}
        |\nabla\tilde{d}_{t-1}(\lambda_t) - \nabla\tilde{d}_{t-1}(\lambda_{t-1})|
        \leq\frac{L_g^2}{\mu} |\lambda_t - \lambda_{t-1}| 
        = \frac{L_g^2}{\mu}\gamma_t |\nabla\tilde{d}_{t-1}(\lambda_{t-1})|,
    \end{equation*}
    where the inequality follows by Corollary \ref{corollary:lipschitz_cont_of_dual} and the equality by the dual update in Alg.~\ref{alg:safe_dual_alg_warm_start}.
    Thus, we can ensure $\nabla\tilde{d}_{t-1}(\lambda_t) \leq 0$ by choosing $\gamma_t$ such that
        $\nabla\tilde{d}_{t-1}(\lambda_{t-1}) + \frac{L_g^2}{\mu}\gamma_t |\nabla\tilde{d}_{t-1}(\lambda_{t-1})| \leq 0,$ 
    which induces the following upper bound on $\gamma_t$ (recall that $\nabla \td_{t-1}(\lambda_{t-1}) \leq 0$):
    \begin{align}
      \gamma_t \leq \frac{-\nabla\tilde{d}_{t-1}(\lambda_{t-1})}{\frac{L_g^2}{\mu} |\nabla\tilde{d}_{t-1}(\lambda_{t-1})|} = \frac{\mu}{L_g^2}.
    \end{align}
    \textbf{The Danger Phase:} $\nabla\tilde{d}_{t-1}(\lambda_{t-1}) > 0$.
    Following Lemma \ref{lemma:local_str_conc_of_dual}, the dual function is locally $\mu_d$-strongly concave $\forall t\in[T]$, namely
       $ \td_t(\tlambda_t^*) - \td_t(\lambda) \geq \frac{\mu_d}{2} (\tlambda_t^* - \lambda)^2, \quad \forall \lambda: g_t(x_{t,\lambda}^*) + \delta \geq -G/2.$
    Moreover, since $\td_t(\lambda)$ is concave $\forall t\in[T]$, we have, for step $t-1$, $\forall \lambda: g_{t-1}(x_{t-1,\lambda}^*) + \delta \geq -G/2$:
    \begin{equation}
         \langle \nabla \td_{t-1}(\lambda), \tlambda_{t-1}^* - \lambda \rangle \geq \td_{t-1}(\tlambda_{t-1}^*) - \td_{t-1}(\lambda) \geq \frac{\mu_d}{2} (\tlambda_{t-1}^* - \lambda)^2.
    \end{equation}
    Now, note that since $\nabla \td_{t-1}(\lambda)$ is monotonically non-increasing and since $\nabla \td_{t-1}(\tlambda_{t-1}^*) = 0$ and $\nabla \td_{t-1}(\lambda_{t-1}) > 0$, we have that $\lambda_{t-1} \leq \tlambda_{t-1}^*$. Note that $\nabla \td_{t-1}(\lambda_{t-1}) > 0$ also implies $g_{t-1}(x_{t-1,\lambda_{t-1}}^*) + \delta > 0 \geq -G/2$. Thus:
    
    \begin{align}
        \langle \nabla\td_{t-1}(\lambda_{t-1}), \tlambda_{t-1}^* - \lambda_{t-1} \rangle &\geq \frac{\mu_d}{2} (\tlambda_{t-1}^* - \lambda_{t-1})^2 \\
        \nabla\td_{t-1}(\lambda_{t-1}) & \geq \frac{\mu_d}{2} (\tlambda_{t-1}^* - \lambda_{t-1}) \\
        \tlambda_{t-1}^* &\leq \lambda_{t-1} + \frac{2}{\mu_d}\nabla\td_{t-1}(\lambda_{t-1})
    \end{align}
    Now, by Lemma \ref{lemma:safety_in_terms_of_dual}, to ensure safety, we need to choose $\lambda_t$ such that $\nabla\td_{t-1}(\lambda_t) \leq 0$. This is equivalent to choosing $\lambda_t \geq \tlambda_{t-1}^*$ since $\nabla\td_{t-1}(\lambda)$ is monotonically non-increasing and $\nabla\td_{t-1}(\tlambda_{t-1}^*) = 0$. Using the dual update rule in Alg.~\ref{alg:safe_dual_alg_warm_start}, we need to choose $\lambda_t$ such that:
    \begin{align}
        \lambda_t &= \lambda_{t-1} + \gamma_t \nabla\td_{t-1}(\lambda_{t-1})
        \geq \lambda_{t-1} + \frac{2}{\mu_d}\nabla\td_{t-1}(\lambda_{t-1})
        \geq \tlambda_{t-1}^*
    \end{align}
    which, since $\nabla \td_{t-1}(\lambda_{t-1}) > 0$, induces the following lower bound on $\gamma_t$:
       $ \gamma_t \geq \frac{2}{\mu_d}.$
\end{proof}

\subsection{Proof of Lemma \ref{lemma:primal_dual_regret_relation}: Relation Between Primal Regret and Dual Regret}\label{appendix:primal_dual_regret_relation}
\begin{proof}
    First, we define the Lagrangian regret as:
    \begin{equation}
        \sum_{t=1}^T \TLag_t(x_t,\tlambda_t^*) - \TLag_t(x_t^*,\lambda_t).
    \end{equation}
    We prove Lemma \ref{lemma:primal_dual_regret_relation} by deriving an upper bound and a lower bound on Lagrangian regret and then combining them.\\
    \textbf{The upper bound:}
    \begin{align}
        \sum_{t=1}^T \TLag_t(x_t,\tlambda_t^*) - \TLag_t(x_t^*,\lambda_t) &= \sum_{t=1}^{T} \TLag_t(x_t,\tlambda_t^*) - \TLag_t(x_t,\lambda_t) + \TLag_t(x_t,\lambda_t) - \TLag_t(x_{t,\lambda_t}^*,\lambda_t) + \\
        &\quad\quad +\TLag_t(x_{t,\lambda_t}^*,\lambda_t) - \TLag_t(x_t^*,\lambda_t) \\
        &\overset{(1)}{\leq} \sum_{t=1}^{T} \TLag_t(x_t,\tlambda_t^*) - \TLag_t(x_t,\lambda_t) + \TLag_t(x_t,\lambda_t) - \TLag_t(x_{t,\lambda_t}^*,\lambda_t) \\
        &\overset{(2)}{=} \sum_{t=1}^{T} (\tlambda_t^* - \lambda_t) (g_t(x_t)+\delta) + \TLag_t(x_t,\lambda_t) - \TLag_t(x_{t,\lambda_t}^*,\lambda_t) \\
        &\overset{(3)}{\leq} \sum_{t=1}^{T} (\tlambda_t^* - \lambda_t) (g_t(x_t)+\delta) + \frac{M}{2} \sum_{t=1}^{T} \norm{x_t - x_{t,\lambda_t}^*}^2 \\
        &\overset{(4)}{\leq} \sum_{t=1}^{T} (\tlambda_t^* - \lambda_t) (g_t(x_t)+\delta) + \frac{MR^2}{2} \\
        &\quad\quad+ \frac{M}{\mu} \sum_{t=2}^{T} \left(\max_{x\in\mathcal{X}} |f_t(x) - f_{t-1}(x)| + \hlambda\delta\right) \\
        &\overset{(5)}{\leq} \sum_{t=1}^{T} (\tlambda_t^* - \lambda_t) (g_t(x_t)+\delta) + \frac{M \hlambda}{\mu}V_{g,T} + \frac{M}{\mu}V_{f,T} + \frac{MR^2}{2}\\
        &\overset{(6)}{\leq} \sum_{t=1}^{T} (\tlambda_t^* - \lambda_t) \left(g_t(x^*_{t,\lambda_t})+\delta + L_g\norm{x^*_{t,\lambda_t}-x^*_{t-1,\lambda_t}}\right)\\
        &\quad\quad + \frac{M \hlambda}{\mu}V_{g,T} + \frac{M}{\mu}V_{f,T} + \frac{MR^2}{2}\\
        &\overset{(7)}{\leq} \sum_{t=1}^{T} (\tlambda_t^* - \lambda_t) (\nabla\td_t(\lambda_t)+\hdelta_t-\delta) + \frac{M \hlambda}{\mu}V_{g,T} + \frac{M}{\mu}V_{f,T} + \frac{MR^2}{2}\\
        &\leq \Ord\left(\hat{R}_{\td}(T)\right) + \frac{M \hlambda}{\mu}V_{g,T} + \frac{M}{\mu}V_{f,T} + \frac{MR^2}{2}
    \end{align}
    where (1) follows since by Eq.~(\ref{eq:x_t_lambda_opt}), $x_{t,\lambda_t}^* = \arg\min_{x\in\mathcal{X}} \TLag_t(x,\lambda_t)$, which implies that $\TLag_t(x_{t,\lambda_t}^*,\lambda_t) - \TLag_t(x_t^*,\lambda_t) \leq 0$, (2) follows by decomposing $\TLag_t$, (3) follows by Lemma \ref{lemma:smooth_str_conv_lagrangian} (smoothness of Lagrangian), (4) follows by Lemma \ref{lemma:bound_distance_opt_x_diff_t} since $x_t = x_{t-1,\lambda_t}^*$ by definition of $x_t$ in Alg.~\ref{alg:safe_dual_alg_warm_start} and by Assumption \ref{assum:bounded_set} (bounded set), (5) follows by $V_{g,T}=\delta T$ and Assumption \ref{assum:bounded_TV_obj} (bounded total variation of the loss), (6) follows since by Assumption \ref{assum:constr_conv_lipsc}) and since $x_t = x^*_{t-1,\lambda_t}$ by Alg.~\ref{alg:safe_dual_alg_warm_start}, (7) follows by definition of the dual gradient and by Lemma \ref{lemma:bound_distance_opt_x_diff_t} and by the definition of $\hdelta_t$ in Lemma \ref{lemma:slowly_changing_dual_gradients}.
    \\
    \textbf{The lower bound:}
    \begin{align}
        \sum_{t=1}^T \TLag_t(x_t,\tlambda_t^*) - \TLag_t(x_t^*,\lambda_t) &= \sum_{t=1}^T f_t(x_t) + \tlambda_t^*(g_t(x_t)+\delta) - f_t(x_t^*) - \lambda_t(g_t(x_t^*)+\delta) \\
        &\geq \sum_{t=1}^T f_t(x_t) + \tlambda_t^*(g_t(x_t)+\delta) - f_t(x_t^*) - \lambda_t \delta
    \end{align}
    where the inequality follows since $g_t(x_t^*) \leq 0$ its definition in Eq. (\ref{eq:primal_baseline}).\\
    Combining the two bounds yields the following:
    \begin{align}
        \sum_{t=1}^T f_t(x_t) - f_t(x_t^*) &\leq \left(\frac{M}{\mu}+ 1\right)\hlambda V_{g,T} + \frac{M}{\mu}V_{f,T} + \frac{M R^2}{2} + \Ord\left(\hat{R}_{\td}(T)\right)
        - \sum_{t=1}^T \tlambda_t^*(g_t(x_t)+\delta)\\
        &\leq \Ord(V_{g,T}) + \Ord(V_{f,T})+ \Ord\left(\hat{\Reg}_{\td}(T)\right) - \sum_{t=1}^T \tlambda_t^*(g_t(x_t)+\delta)
    \end{align}
    where the last inequality follows from the regret analysis in Theorem \ref{theorem:regret_analyis_our_method}.
    Now let us consider the sum on the right. Similarly to the analysis in Theorem \ref{theorem:regret_analyis_our_method}, we set a new counter $\tau$ which resets after every phase. Recall that each safe phase $i$ lasts $\T^S_i$ steps and each danger phase $j$ lasts $\T^D_j$ steps. We have:
    \begin{align}
        &\quad\sum_{t=1}^T -\tlambda_t^*(g_t(x_t)+\delta)\\ &\overset{(1)}{\leq} \hlambda \sum_{t=1}^T -g_t(x_t) -\hlambda\delta T
        \leq \hlambda \sum_{t=1}^T -g_t(x_t)
        \overset{(2)}{\leq} \hlambda \sum_{t=1}^T (-g_{t-1}(x_t) + \delta)\\
        &\overset{(3)}{\leq} \hlambda \sum_{t=1}^T (-\nabla \td_{t-1}(\lambda_t) + 2\delta)
        \overset{(4)}{\leq} 2\hlambda V_{g,T} + \hlambda \sum_{t=1}^T (\underbrace{-\nabla \td_t(\lambda_t)}_{z_t} + \hdelta_t) \\
        &= 2\hlambda V_{g,T} + \hlambda\sum_{t=1}^T \hdelta_t + \hlambda \sum_{i=1}^n\sum_{\tau=1}^{\T^S_i} z_{\tau} + \hlambda\sum_{j=1}^m\sum_{\tau=1}^{\T^D_j} z_{\tau}\\
        &\overset{(5)}{\leq} 2\hlambda V_{g,T} + \hlambda\sum_{t=1}^T \hdelta_t + \hlambda \sum_{i=1}^n\sum_{\tau=1}^{\T^S_i} z_{\tau}
        \overset{(6)}{\leq} 2\hlambda V_{g,T} + \hlambda\sum_{t=1}^T \hdelta_t + \hlambda \sum_{i=1}^{n} \frac{6L_g^2}{\mu\mu_d} \sum_{\tau=1}^{\T^S_i+1}\hdelta_{\tau} \\
        &\overset{(7)}{\leq} 2\hlambda V_{g,T} + \hlambda\frac{7L_g^2}{\mu\mu_d} \sum_{t=1}^T\hdelta_t \overset{(8)}{\leq} 2\hlambda V_{g,T} + \hlambda\frac{7L_g^2}{\mu\mu_d} \sum_{t=1}^T \left(\delta + L_g\sqrt{\frac{2}{\mu} \left(\max_{x\in\mathcal{X}}|f_{t} - f_{t-1}| + \hlambda\delta\right)}\right)\\
        &\overset{(9)}{\leq} 2\hlambda V_{g,T} + \hlambda\frac{7L_g^2}{\mu\mu_d}V_{g,T} + \hlambda\frac{7L_g^3}{\mu\mu_d}\sqrt{\frac{2}{\mu}} \sum_{t=1}^{T} \left(\sqrt{\max_{x\in\mathcal{X}}|f_{t} - f_{t-1}|} + \sqrt{\hlambda\delta}\right) \\
        &\overset{(10)}{\leq} 2\hlambda V_{g,T} + \hlambda\frac{7L_g^2}{\mu\mu_d}V_{g,T} + \hlambda\frac{7L_g^3}{\mu\mu_d}\sqrt{\frac{2}{\mu}} \left(\sqrt{T \sum_{t=1}^T \max_{x\in\mathcal{X}}|f_{\tau} - f_{\tau-1}|} + \sqrt{\hlambda\delta}T\right) \\
        &\overset{(11)}{\leq} 2\hlambda V_{g,T} + \hlambda\frac{7L_g^2}{\mu\mu_d} V_{g,T} + \hlambda\frac{7L_g^3}{\mu\mu_d}\sqrt{\frac{2}{\mu}}\sqrt{TV_{f,T}} + \hlambda\frac{7L_g^3}{\mu\mu_d}\sqrt{\frac{2\hlambda}{\mu}}\sqrt{TV_{g,T}} \\
        &= \Ord\left(\sqrt{V_{g,T}T}\right) + \Ord\left(\sqrt{V_{f,T}T}\right)
    \end{align}
    where (1) follows from Lemma \ref{lemma:universal_bound_dual_baseline} and since $g_t(x_t)\leq0$ by Theorem \ref{theorem:safe_stepsize_bounds}, (2) follows by Assumption \ref{assum:slowly_changing_constr}, (3) follows by the definition of the dual function, (4) follows by Lemma $\ref{lemma:slowly_changing_dual_gradients}$ and $V_{g,T}=\delta T$, (5) follows since in the danger phase $\nabla \td_{\tau}(\lambda_{\tau}) > 0$ and thus $z_{\tau} < 0, \forall \tau\in[\T^D_j], \forall j\in[m]$, (6) follows from property (B) in Lemma \ref{lemma:helpful_properties_for_case1}, (7) follows since $L^2_g/\mu \geq \mu_d$ and since $\sum_{t=1}^T \hdelta_t \geq \sum_{i=1}^n \sum_{\tau=1}^{\T^S_i+1}\hdelta_{\tau}$, since $\delta_t \geq 0, \forall t\in[T]$, (8) follows from the definition of $\hdelta_t$ in Lemma \ref{lemma:slowly_changing_dual_gradients}, (9) follows since $\forall X,Y\geq0: \sqrt{X+Y} \leq \sqrt{X} + \sqrt{Y}$, (10) follows by Jensen's inequality, and (11) follows by Assumption \ref{assum:bounded_TV_obj} and $V_{g,T}=\delta T$.\\
    \textbf{Putting it all together.} We have:
    \begin{align}
        \sum_{t=1}^T f_t(x_t) - f_t(x_t^*) &= \Ord(V_{g,T}) + \Ord(V_{f,T}) + \Ord\left(\hat{\Reg}_{\td}(T)\right) + \Ord\left(\sqrt{V_{g,T}T}\right) + \Ord\left(\sqrt{V_{f,T}T}\right)\\
        &=\Ord\left(\max\left\{\hat{\Reg}_{\td}(T), \sqrt{(V_{g,T}+V_{f,T})T}\right\}\right)
    \end{align}
    where the second equality follows since $V_{f,T} = o(T)$ and $V_{g,T} = o(T)$, and thus our proof is concluded.
\end{proof}

\subsection{Proof of Lemma \ref{lemma:slowly_changing_dual_gradients}: Slowly Changing Dual Gradients}\label{appendix:slowly_changing_dual_gradients}
To prove Lemma \ref{lemma:slowly_changing_dual_gradients}, we first prove another helpful lemma.
\begin{lemma}\label{lemma:bound_distance_opt_x_diff_t}
    Under Assumptions \ref{assum:bounded_set}-\ref{assum:slowly_changing_constr}, \ref{assum:non_shallow_constr_and_strong_duality}, the distance between $x_{t,\lambda}^* = \arg\min_{x\in\mathcal{X}}\TLag_t(x,\lambda)$ and  $x_{t-1,\lambda}^* = \arg\min_{x\in\mathcal{X}}\TLag_{t-1}(x,\lambda)$ is bounded as follows, $\forall \lambda \geq 0$:
    \begin{equation}
        \norm{x_{t,\lambda}^* - x_{t-1,\lambda}^*} \leq \sqrt{\frac{2}{\mu}\left(\max_{x\in\mathcal{X}} |f_t(x) - f_{t-1}(x)| + \lambda\delta\right)}
    \end{equation}
\end{lemma}
\begin{proof}
    By Lemma \ref{lemma:smooth_str_conv_lagrangian} on the strong convexity of $\TLag$,
    \begin{align}
        &\TLag_t(x_{t-1,\lambda}^*,\lambda) \geq \TLag_t(x_{t,\lambda}^*,\lambda) + \langle \nabla_x\TLag_t(x_{t,\lambda}^*,\lambda), x_{t-1,\lambda}^* - x_{t,\lambda}^* \rangle + \frac{\mu}{2}\norm{x_{t,\lambda}^* - x_{t-1,\lambda}^*}^2 \\
        &\TLag_{t-1}(x_{t,\lambda}^*,\lambda) \geq \TLag_{t-1}(x_{t-1,\lambda}^*,\lambda) + \langle \nabla_x\TLag_{t-1}(x_{t-1,\lambda}^*,\lambda), x_{t,\lambda}^* - x_{t-1,\lambda}^* \rangle + \frac{\mu}{2}\norm{x_{t,\lambda}^* - x_{t-1,\lambda}^*}^2
    \end{align}
    By definition of the optimal points $x_{t-1,\lambda}^*$ and $x_{t,\lambda}^*$,
    \begin{align}
        &\TLag_t(x_{t-1,\lambda}^*,\lambda) \geq \TLag_t(x_{t,\lambda}^*,\lambda) + \frac{\mu}{2}\norm{x_{t,\lambda}^* - x_{t-1,\lambda}^*}^2 \\
        &\TLag_{t-1}(x_{t,\lambda}^*,\lambda) \geq \TLag_{t-1}(x_{t-1,\lambda}^*,\lambda) + \frac{\mu}{2}\norm{x_{t,\lambda}^* - x_{t-1,\lambda}^*}^2
    \end{align}
    Decomposing the Lagrangians,
    \begin{align}
        &f_t(x_{t-1,\lambda}^*) + \lambda (g_t(x_{t-1,\lambda}^*)+\delta) \geq f_t(x_{t,\lambda}^*) + \lambda (g_t(x_{t,\lambda}^*)+\delta) + \frac{\mu}{2}\norm{x_{t,\lambda}^* - x_{t-1,\lambda}^*}^2 \\
        &f_{t-1}(x_{t,\lambda}^*) + \lambda (g_{t-1}(x_{t,\lambda}^*)+\delta) \geq f_{t-1}(x_{t-1,\lambda}^*) + \lambda (g_{t-1}(x_{t-1,\lambda}^*)+\delta) + \frac{\mu}{2}\norm{x_{t,\lambda}^* - x_{t-1,\lambda}^*}^2
    \end{align}
    Summing and rearranging,
    \begin{align}
        \mu\norm{x_{t,\lambda}^* - x_{t-1,\lambda}^*}^2 &\leq f_t(x_{t-1,\lambda}^*) - f_{t-1}(x_{t-1,\lambda}^*) + f_{t-1}(x_{t,\lambda}^*) - f_t(x_{t,\lambda}^*) +  \\
        &\quad + \lambda(g_{t-1}(x_{t,\lambda}^*) - g_t(x_{t,\lambda}^*)) + \lambda(g_t(x_{t-1,\lambda}^*) - g_{t-1}(x_{t-1,\lambda}^*)) \\
        &\leq 2\max_{x\in\mathcal{X}} |f_t(x) - f_{t-1}(x)| + 2\lambda \max_{x\in\mathcal{X}}|g_t(x) - g_{t-1}(x)| \\ 
        &\leq 2\max_{x\in\mathcal{X}} |f_t(x) - f_{t-1}(x)| + 2\lambda\delta
    \end{align}
    where the last inequality follows from Assumption \ref{assum:slowly_changing_constr} (slowly changing constraint) and Lemma \ref{lemma:universal_bound_dual_baseline}.
    Dividing by $\mu$ taking the square root concludes the proof.
\end{proof}
Now we prove Lemma \ref{lemma:slowly_changing_dual_gradients}.
\begin{proof} Proof of Lemma \ref{lemma:slowly_changing_dual_gradients}.
    \begin{align}
        &\max_{\lambda>0}{\norm{\nabla \td_t(\lambda) - \nabla \td_{t-1}(\lambda)}}\\
        =&\max_{\lambda>0}\left(|g_t(x_{t,\lambda}^*) + \delta - (g_{t-1}(x_{t-1,\lambda}^*)+\delta)|\right)\\
        \leq&\max_{\lambda>0}\left(|g_t(x_{t,\lambda}^*) - g_{t-1}(x_{t,\lambda}^*)| + |g_{t-1}(x_{t,\lambda}^*) - g_{t-1}(x_{t-1,\lambda}^*)|\right)\\
        \leq&\max_{\lambda>0}\left(|g_t(x_{t,\lambda}^*) - g_{t-1}(x_{t,\lambda}^*)| + L_g\norm{x_{t,\lambda}^* - x_{t-1,\lambda}^*}\right)\\
        \leq&\max_{\lambda>0}\left(\delta + L_g\norm{x_{t,\lambda}^* - x_{t-1,\lambda}^*}\right)\\
        \leq&\max_{\lambda>0}\left(\delta + L_g\sqrt{\frac{2}{\mu}\left(\max_{x\in\mathcal{X}} |f_t(x) - f_{t-1}(x)| + \lambda\delta\right)}\right)\\
        \leq &\delta + L_g\sqrt{\frac{2}{\mu}\left(\max_{x\in\mathcal{X}} |f_t(x) - f_{t-1}(x)| + \hlambda\delta\right)}
    \end{align}
    where the equality follows by definition of the dual gradients, the first inequality is by the triangle inequality, the second is by Assumption \ref{assum:constr_conv_lipsc} (Lipschitz continuity of the constraints), the third follows by Assumption \ref{assum:slowly_changing_constr} (slowly changing constraints), the fourth follows by Lemma \ref{lemma:bound_distance_opt_x_diff_t}, and the last follows by Lemma \ref{lemma:universal_bound_dual_baseline}.
\end{proof}

\subsection{Proof of Corollary \ref{corollary:bounded_distance_between_lambda_tlambda}: Bounded Distance between Dual Optimal Values:}\label{appendix:bouned_distance_between_lambda_tlambda}
\begin{proof}
    Since the dual function $\td_t(\lambda)$ is concave, and by Lemma~\ref{lemma:local_str_conc_of_dual} it is also locally $\mu_d$-strongly concave, we have:
    \begin{equation}
        \langle \nabla\td_t(\lambda), \tlambda_t^* - \lambda \rangle \geq \td_t(\tlambda_t^*) - \td_t(\lambda) \geq \frac{\mu_d}{2} (\tlambda_t^* - \lambda)^2
    \end{equation}
    Thus:
    \begin{equation}
        \langle \nabla\td_t(\tlambda_{t-1}^*), \tlambda_t^* - \tlambda_{t-1}^* \rangle \geq \frac{\mu_d}{2} (\tlambda_t^* - \tlambda_{t-1}^*)^2
    \end{equation}
    Now, if $\tlambda_{t-1}^* \leq \tlambda_t^*$, this implies that $\nabla \td_t(\tlambda_{t-1}^*) \geq 0$ since the dual gradients are monotonically non-increasing and $\nabla \td_t(\tlambda_t^*) = 0$, and thus:
    \begin{equation}
        |\tlambda_t^* - \tlambda_{t-1}^*| \leq \frac{2}{\mu_d}\nabla \td_t(\tlambda_{t-1}^*) \leq \frac{2}{\mu_d}(\nabla \td_{t-1}(\tlambda_{t-1}^*)+\hdelta_t) = \frac{2\hdelta_t}{\mu_d}
    \end{equation}
    where the second inequality is by Lemma \ref{lemma:slowly_changing_dual_gradients} and the last is since $\nabla \td_{t-1}(\tlambda_{t-1}^*) = 0$, by definition.
    
    \noindent Alternatively, if $\tlambda_{t-1}^* \geq \tlambda_t^*$, this implies that $\nabla \td_t(\tlambda_{t-1}^*) \leq 0$, and thus similarly:
    \begin{equation}
        |\tlambda_t^* - \tlambda_{t-1}^*| \leq -\frac{2}{\mu_d}\nabla \td_t(\tlambda_{t-1}^*) \leq -\frac{2}{\mu_d}(\nabla \td_{t-1}(\tlambda_{t-1}^*)-\hdelta_t) = \frac{2\hdelta_t}{\mu_d}
    \end{equation}
    Thus, in total:
    \begin{equation}
        |\tlambda_t^* - \tlambda_{t-1}^*| \leq \frac{2\hdelta_t}{\mu_d}
    \end{equation}
\end{proof}

\subsection{Proof of Lemma \ref{lemma:helpful_properties_for_case1}}\label{appendix:proof_lemma_helpful_properties}
We state and prove each of the properties in Lemma \ref{lemma:helpful_properties_for_case1}.

\noindent Consider the $i$'th safe phase. For convenience, and similar to the previous analyses, we set a new counter $\tau=1,2,...,\T^S_i$ for this phase. Let $\lambda_1$ be the initial iterate of this phase 
and recall that we denote $z_\tau = -\nabla \td_{\tau}(\lambda_{\tau})$ and that during safe phases we use $\gamma_t = \mu/L_g^2$. Using Alg.~\ref{alg:safe_dual_alg_warm_start} we have:

\begin{itemize}
    \item[(A)] $\sum_{\tau=1}^{\T^S_i} z_{\tau} \leq \frac{\hlambda L_g^2}{\mu}$.

    \item[(B)] $\sum_{\tau=1}^{\T^S_i} z_{\tau} \leq 
    \frac{6L_g^2}{\mu\mu_d} \sum_{\tau=1}^{\T^S_i+1}\hdelta_{\tau}$.
\end{itemize}
\begin{proof} \textbf{(A). }
    $\sum_{\tau=1}^{\T^S_i} z_{\tau} = \sum_{\tau=1}^{\T^S_i} \frac{1}{\gamma_t}(\lambda_{\tau} - \lambda_{\tau+1}) = \frac{L_g^2}{\mu}(\lambda_1 - \lambda_{\T^S_i+1}) \leq \frac{L_g^2}{\mu}\lambda_1 \leq \frac{\hlambda L_g^2}{\mu}$.
\end{proof}

\begin{proof} \textbf{(B). }
    By Corollary \ref{corollary:lipschitz_cont_of_dual} on the Lipschitz continuity of $\nabla \td_{\tau}$:
\begin{align}
    &z_{\tau} = -\nabla\td_{\tau}(\lambda_{\tau}) \leq -\nabla \td_{\tau}(\lambda_1) + \frac{L_g^2}{\mu} (\lambda_{\tau} - \lambda_1) \\
    &\lambda_1 - \lambda_{\tau} \leq \frac{-\nabla\td_{\tau}(\lambda_1) - z_{\tau}}{L_g^2/\mu}
\end{align}
Using (A):
\begin{align}
    \sum_{\tau=1}^{\T^S_i} z_{\tau} = \frac{L_g^2}{\mu}(\lambda_1 - \lambda_{\T^S_i+1}) 
    &\leq -\nabla\td_{\T^S_i+1}(\lambda_1) - z_{\T^S_i+1}\\
    &\leq -\nabla\td_{\T^S_i+1}(\lambda_1) + \hdelta_{\T^S_i+1}
\end{align}
where the last inequality follows since:
\begin{equation}
    -z_{\T^S_i+1} = \nabla\td_{\T^S_i+1}(\lambda_{\T^S_i+1}) \leq \nabla\td_{\T^S_i}(\lambda_{\T^S_i+1}) + \hdelta_{\T^S_i+1} \leq \hdelta_{\T^S_i+1}
\end{equation}
where the first inequality follows by Lemma \ref{lemma:slowly_changing_dual_gradients} and the second by Lemma \ref{lemma:safety_in_terms_of_dual} since Alg. \ref{alg:safe_dual_alg_warm_start} ensures safety.
Now, following Corollary \ref{corollary:lipschitz_cont_of_dual}:
\begin{align}
        |\nabla \td_{\T^S_i+1}(\lambda_1) - \nabla \td_{\T^S_i+1}(\tlambda_{\T^S_i+1}^*)| &\leq \frac{L_g^2}{\mu} |\lambda_1 - \tlambda_{\T^S_i+1}^*| \\
        &\leq \frac{L_g^2}{\mu} \left(|\tlambda_1^* - \tlambda_{\T^S_i+1}^*| + |\lambda_1 - \tlambda_1^*|\right) \\
        &\leq \frac{L_g^2}{\mu} \left(\sum_{\tau=1}^{\T^S_i}|\tlambda_{\tau}^* - \tlambda_{\tau+1}^*| + |\lambda_1 - \tlambda_1^*|\right)\\
        &\leq \frac{L_g^2}{\mu} \left(\sum_{\tau=1}^{\T^S_i} \frac{2\hdelta_{\tau+1}}{\mu_d} + |\lambda_1 - \tlambda_1^*|\right)
    \end{align}
where the second and third inequalities follow from the triangle inequality and the fourth follows from Corollary \ref{corollary:bounded_distance_between_lambda_tlambda}.
To bound $|\lambda_1 - \tlambda_1^*|$, note that if the $i$'th safe phase occurs after a danger phase, then $\lambda_1$ is the last iterate of the previous danger phase, and thus by Eq.~(\ref{eq:case_2_iterate_bound_pt2}), it is bounded as $|\lambda_1 - \tlambda_1^*| \leq 6\hdelta_1/\mu_d$. Otherwise, if the first phase is a safe phase, then thanks to the warm start of Alg.~\ref{alg:safe_dual_alg_warm_start} using the strong oracle, we have that $(x_1,\lambda_1)$ is the primal-dual solution of the optimization problem $\arg\min_{x\in\mathcal{X}} f_1(x) \; \text{s.t.}\; g_1(x)+\delta \leq 0$, and thus $\nabla\td_1(\lambda_1)=0$ which implies that $|\lambda_1 - \tlambda_1^*|=0$. Thus in total, we have:
\begin{equation}
    |\nabla \td_{\T^S_i+1}(\lambda_1) - \nabla \td_{\T^S_i+1}(\tlambda_{\T^S_i+1}^*)| \leq \frac{L_g^2}{\mu} \left(\sum_{\tau=1}^{\T^S_i} \frac{2\hdelta_{\tau+1}}{\mu_d} + 6\frac{\hdelta_1}{\mu_d}\right)
\end{equation}
Now, using the fact that $\nabla \td_{\T^S_i+1}(\tlambda_{\T^S_i+1}^*) = 0$ by definition of $\tlambda_{\T^S_i+1}^*$, we have:
\begin{align}
    \sum_{\tau=1}^{\T^S_i} z_{\tau} &\leq -\nabla\td_{\T^S_i+1}(\lambda_1) + \hdelta_{\T^S_i+1} \\
    &\leq \frac{L_g^2}{\mu} \left(\sum_{\tau=1}^{\T^S_i} \frac{2\hdelta_{\tau+1}}{\mu_d} + 6\frac{\hdelta_1}{\mu_d}\right) + \hdelta_{\T^S_i+1}\\
    &\leq \frac{6L_g^2}{\mu} \sum_{\tau=1}^{\T^S_i+1} \frac{\hdelta_{\tau}}{\mu_d}
\end{align}
where the last inequality follows since $L^2_g/\mu \geq \mu_d$. That is because $L^2_g/\mu$ is an upper bound on the curvature of the dual function by Corollary \ref{corollary:lipschitz_cont_of_dual} while $\mu_d$ is a lower bound by Lemma \ref{lemma:local_str_conc_of_dual}.
\end{proof}

\newpage
\subsection{Full Proof of Theorem \ref{theorem:regret_analyis_our_method}}\label{appendix:full_regret_analysis_our_method}
\begin{proof}
    We now analyze and bound the dual regret in each phase separately, then we use these bounds to bound the primal regret using Lemma \ref{lemma:primal_dual_regret_relation}.
    \paragraph{The Danger Phase.} 
    We analyze the total dual regret incurred during all $m$ danger phases. We do so by first bounding the single-step regret at some step $t$ during any danger phase, defined as:
    \begin{equation}
        r_{\td,t} = \td_t(\tlambda_t^*) - \td_t(\lambda_t).
    \end{equation}
    Note that, by definition of the "danger phase", $\nabla \td_{t-1}(\lambda_{t-1}) > 0$, and thus by Theorem~\ref{theorem:safe_stepsize_bounds}, we use $\gamma_t = 2/\mu_d$ in the dual updates in Alg.~\ref{alg:safe_dual_alg_warm_start}.
    Additionally, note that by Lemma \ref{lemma:slowly_changing_dual_gradients}, for any step $t$, $\nabla \td_t(\lambda_t) \leq \nabla \td_{t-1}(\lambda_t) + \hdelta_t \leq \hdelta_t$,
    where the second inequality follows from Lemma \ref{lemma:safety_in_terms_of_dual} since Alg.~\ref{alg:safe_dual_alg_warm_start} ensures safety by Theorem \ref{theorem:safe_stepsize_bounds}.
    Now, we bound the single-step regret:
    \begin{align}\label{eq:app_single_step_dual_reg_bound}
        r_{\td,t} &= \td_t(\tlambda_t^*) - \td_t(\lambda_t) \leq \langle \nabla\td_t(\lambda_t), \tlambda_t^* - \lambda_t \rangle\leq |\nabla \td_t(\lambda_t)| \cdot |\tlambda_t^* - \lambda_t| \leq \hdelta_t |\tlambda_t^* - \lambda_t|,
    \end{align}
    where the first inequality is due to the concavity of $\td_t(\lambda)$, the second is by the Cauchy-Schwartz inequality, and the third is since $0<\nabla \td_t(\lambda_t)\leq\hdelta_t$.
    Now, before bounding $|\tlambda_t^* - \lambda_t|$, note that $\nabla \td_{t-1}(\lambda_{t-1}) > 0$, by definition of the "danger phase", implies $\lambda_{t-1} \leq \tlambda_{t-1}^*$ since $\nabla \td_{t-1}(\lambda)$ is monotonically non-increasing and $\nabla \td_{t-1}(\tlambda_{t-1}^*) = 0$. Moreover, the safety criterion in Lemma \ref{lemma:safety_in_terms_of_dual} implies that $\nabla \td_{t-2}(\lambda_{t-1}) \leq 0$ which similarly implies $\lambda_{t-1} \geq \tlambda_{t-2}^*$. Thus, in total we have $\tlambda_{t-2}^* \leq \lambda_{t-1} \leq \tlambda_{t-1}^*$.
    Now, we bound $|\tlambda_t^* - \lambda_t|$:
    \begin{align}
        |\tlambda_t^* - \lambda_t| &\overset{(1)}{=} \left|\tlambda_t^* - (\lambda_{t-1} + \frac{2}{\mu_d} \nabla\td_{t-1}(\lambda_{t-1}))\right|\label{eq:app_case_2_iterate_bound_pt1}
        \overset{(2)}{\leq} |\tlambda_t^* - \lambda_{t-1}| + \frac{2}{\mu_d} |\nabla\td_{t-1}(\lambda_{t-1})| 
        \\
        &
        \overset{(3)}{\leq} |\tlambda_t^* - \lambda_{t-1}| + \frac{2}{\mu_d} \hdelta_{t-1}
        \overset{(4)}{\leq} |\tlambda_t^* - \tlambda_{t-1}^*| + |\tlambda_{t-1}^* - \lambda_{t-1}| + \frac{2}{\mu_d} \hdelta_{t-1} 
        \\
        &\overset{(5)}{\leq} |\tlambda_t^* - \tlambda_{t-1}^*| + |\tlambda_{t-1}^* - \tlambda_{t-2}^*| + \frac{2}{\mu_d} \hdelta_{t-1}
        \overset{(6)}{\leq} \frac{2}{\mu_d}\hdelta_t + \frac{4}{\mu_d}\hdelta_{t-1} \label{eq:app_case_2_iterate_bound_pt2}
    \end{align}
    where (1) is by the update rule, (2) is by the triangle inequality, (3) is since $0<\nabla \td_{t-1}(\lambda_{t-1})\leq\hdelta_{t-1}$ for any $t$ during any danger phase,
    (4) is by the triangle inequality, 
    (5) is since $\tlambda_{t-2}^* \leq \lambda_{t-1} \leq \tlambda_{t-1}^*$,
    and (6) is by Corollary \ref{corollary:bounded_distance_between_lambda_tlambda}.
    Now, to analyze the total dual regret, we first set a new counter for each danger phase $j$, denoted by $\tau = 1,2,...,\T^D_j$. Note that the counter resets after every phase.
    Thus, the total dual regret incurred during all $m$ danger phases, which we denote by $\Reg_{\td}^D$, is bounded by:
    \begin{align}
        \Reg_{\td}^{D} &= \sum_{j=1}^{m}\sum_{\tau=1}^{\T^{D}_j} r_{\td,\tau}
        \overset{(1)}{\leq} \sum_{j=1}^{m}\sum_{\tau=1}^{\T^{D}_j} \hdelta_\tau|\tlambda_{\tau}^* - \lambda_{\tau}|
        \overset{(2)}{\leq} \frac{2}{\mu_d}\sum_{j=1}^{m}\sum_{\tau=1}^{\T^{D}_j} \hdelta_\tau^2 + 2\hdelta_\tau\hdelta_{\tau-1}\\
        &\overset{(3)}{\leq} \frac{2}{\mu_d}\sum_{j=1}^{m}\sum_{\tau=1}^{\T_j^{D}}\hdelta_\tau^2 + \frac{2}{\mu_d}\left(\sum_{j=1}^{m}\sum_{\tau=1}^{\T_j^{D}}\hdelta_\tau^2 + \sum_{j=1}^{m}\sum_{\tau=1}^{\T_j^{D}}\hdelta_{\tau-1}^2\right)
    \end{align}
    \begin{align}
        &\overset{(4)}{\leq} \frac{4}{\mu_d}\sum_{j=1}^{m}\sum_{\tau=1}^{\T_j^D}\left(\delta + L_g\sqrt{\frac{2}{\mu}\left(\max_{x\in\mathcal{X}}|f_\tau(x) - f_{\tau-1}(x)| + \hlambda\delta\right)}\right)^2 + \\
        &\quad\quad + \frac{2}{\mu_d}\sum_{j=1}^{m}\sum_{\tau=1}^{\T_j^D}\left(\delta + L_g\sqrt{\frac{2}{\mu}\left(\max_{x\in\mathcal{X}}|f_{\tau-1}(x) - f_{\tau-2}(x)| + \hlambda\delta\right)}\right)^2 \\
        &= \frac{4}{\mu_d}\sum_{j=1}^{m}\sum_{\tau=1}^{\T_j^D} \delta^2 + \frac{4}{\mu_d}\sum_{j=1}^{m}\sum_{\tau=1}^{\T_j^D} 2\delta L_g\sqrt{\frac{2}{\mu}}\sqrt{\max_{x\in\mathcal{X}}|f_\tau(x) - f_{\tau-1}(x)| + \hlambda\delta} + \\
        &\quad\quad + \frac{4}{\mu_d}\sum_{j=1}^{m}\sum_{\tau=1}^{\T_j^D} L^2_g\frac{2}{\mu}\left(\max_{x\in\mathcal{X}}|f_\tau(x) - f_{\tau-1}(x)| + \hlambda\delta\right) + \\
        &\quad + \frac{2}{\mu_d}\sum_{j=1}^{m}\sum_{\tau=1}^{\T_j^D} \delta^2 + \frac{2}{\mu_d}\sum_{j=1}^{m}\sum_{\tau=1}^{\T_j^D} 2\delta L_g\sqrt{\frac{2}{\mu}}\sqrt{\max_{x\in\mathcal{X}}|f_{\tau-1}(x) - f_{\tau-2}(x)| + \hlambda\delta} + \\
        &\quad\quad + \frac{2}{\mu_d}\sum_{j=1}^{m}\sum_{\tau=1}^{\T_j^D} L^2_g\frac{2}{\mu}\left(\max_{x\in\mathcal{X}}|f_{\tau-1}(x) - f_{\tau-2}(x)| + \hlambda\delta\right) \\
        &\overset{(5)}{\leq} \frac{4}{\mu_d}\left(\delta V_{g,T} + \sum_{j=1}^{m}\sum_{\tau=1}^{\T_j^D}2\delta L_g\sqrt{\frac{2}{\mu}}\left(\sqrt{\max_{x\in\mathcal{X}}|f_\tau(x) - f_{\tau-1}(x)|} + \sqrt{\hlambda\delta}\right) + L^2_g\frac{2}{\mu}\left(V_{f,T} + \hlambda V_{g,T}\right)\right) + \\
        &\quad + \frac{2}{\mu_d}\left(\delta V_{g,T} + \sum_{j=1}^{m}\sum_{\tau=1}^{\T_j^D}2\delta L_g\sqrt{\frac{2}{\mu}}\left(\sqrt{\max_{x\in\mathcal{X}}|f_{\tau-1}(x) - f_{\tau-2}(x)|} + \sqrt{\hlambda\delta}\right) + L^2_g\frac{2}{\mu}\left(V_{f,T} + \hlambda V_{g,T} \right)\right)\\
        &\overset{(6)}{\leq} \frac{4}{\mu_d}\left(\delta V_{g,T} + 2\delta L_g\sqrt{\frac{2}{\mu}}\left(\sqrt{\sum_{j=1}^{m} T_j^D \cdot \sum_{j=1}^{m}\sum_{\tau=1}^{\T_j^D}\max_{x\in\mathcal{X}}|f_\tau(x) - f_{\tau-1}(x)|} + \sqrt{\hlambda\delta}T\right) + L^2_g\frac{2}{\mu}\left(V_{f,T} + \hlambda V_{g,T}\right)\right) + \\
        &\quad + \frac{2}{\mu_d}\left(\delta V_{g,T} + 2\delta L_g\sqrt{\frac{2}{\mu}}\left(\sqrt{\sum_{j=1}^{m} T_j^D \cdot \sum_{j=1}^{m}\sum_{\tau=1}^{\T_j^D}\max_{x\in\mathcal{X}}|f_{\tau-1}(x) - f_{\tau-2}(x)|} + \sqrt{\hlambda\delta}T\right) + L^2_g\frac{2}{\mu}\left(V_{f,T} + \hlambda V_{g,T} \right)\right)\\
        &\overset{(7)}{=} \frac{6}{\mu_d}\left(\delta V_{g,T} + 2L_g\sqrt{\frac{2}{\mu}}\delta\sqrt{TV_{f,T}} + 2L_g\sqrt{\frac{2\hlambda}{\mu}}\delta\sqrt{TV_{g,T}} + L^2_g\frac{2}{\mu}V_{f,T} + L^2_g\frac{2}{\mu}\hlambda V_{g,T} \right)\\
        &=\Ord\left(\delta V_{g,T}\right) + \Ord\left(\delta\sqrt{V_{f,T}T}\right) + \Ord\left(\delta\sqrt{V_{g,T}T}\right) + \Ord\left(V_{f,T}\right) + \Ord\left(V_{g,T}\right) \\
        &\overset{(8)}{=} \Ord\left(V_{g,T} + V_{f,T}\right)
    \end{align}
    where (1) is by Eq.~(\ref{eq:app_single_step_dual_reg_bound}), (2) is by Eq.~(\ref{eq:app_case_2_iterate_bound_pt1}-\ref{eq:app_case_2_iterate_bound_pt2}), (3) is since $\forall a,b\in\R: 2ab \leq a^2 + b^2$, (4) is by the definition of $\hdelta_t$ in Lemma \ref{lemma:slowly_changing_dual_gradients}, (5) is by $V_{g,T}=\delta T$, Assumption \ref{assum:bounded_TV_obj} (bounded total variation), the fact that $\forall X,Y \geq 0: \sqrt{X+Y} \leq \sqrt{X} + \sqrt{Y}$, and since $\sum_{j=1}^{m}\sum_{\tau=1}^{\T_j^D}1 \leq T$, (6) is by Jensen's inequality, (7) is by $V_{g,T}=\delta T$, Assumption \ref{assum:bounded_TV_obj}, and since $\sum_{j=1}^{m} \T_j^D \leq T$, and finally (8) follows since $V_{f,T}=o(T)$, $V_{g,T}=o(T)$, and $\delta = o(T^{-\alpha})$ with $\alpha>0$, which imply that $\delta\sqrt{V_{g,T}T} = \delta^{3/2}T < \delta T = V_{g,T}$ and similarly $\delta\sqrt{V_{f,T}T} < \delta \sqrt{TT} = \delta T = V_{g,T}$.

    \paragraph{The Safe Phase.} We analyze the total dual regret incurred during all $n$ safe phases, where each safe phase $i$ lasts for $\T^S_i$ steps.
    For convenience, we set a new counter for the steps during each safe phase, denoted by $\tau = 1,2,...,\T^S_i$. The counter resets after every phase. Note that throughout any safe phase $i$, $\forall \tau\in[\T^S_i]$, $\nabla \td_{\tau-1}(\lambda_{\tau-1}) \leq 0$, and thus Alg.~\ref{alg:safe_dual_alg_warm_start} use $\gamma_t = \mu/L^2_g$ in the dual update, which ensures safety by Theorem \ref{theorem:safe_stepsize_bounds}. 
    Throughout this analysis we denote $z_\tau = -\nabla \td_{\tau}(\lambda_{\tau})$.
    Now, we bound the total dual regret incurred during all $n$ safe phases, which we denote by $\Reg_{\td}^S$:
    \begin{align*}
        \Reg_{\td}^S &= \sum_{i=1}^{n}\sum_{{\tau}=1}^{\T^S_i} \td_{\tau}(\tlambda_{\tau}^*) - \td_{\tau}(\lambda_{\tau}) 
        \overset{(1)}{\leq} \sum_{i=1}^{n}\sum_{{\tau}=1}^{\T^S_i} \langle \nabla \td_{\tau}(\lambda_{\tau}), \tlambda_{\tau}^* - \lambda_{\tau} \rangle -\frac{\mu_d}{2} |\lambda_{\tau} - \tlambda_{\tau}^*|^2 \\
        &= \sum_{i=1}^{n}\sum_{{\tau}=1}^{\T^S_i}\left( -\frac{1}{2}\left| \sqrt{\mu_d}(\tlambda_{\tau}^* - \lambda_{\tau}) - \frac{1}{\sqrt{\mu_d}}\nabla \td_{\tau}(\lambda_{\tau}) \right|^2 + \frac{1}{2\mu_d}|\nabla \td_{\tau}(\lambda_{\tau})|^2\right)\\
        &\leq \frac{1}{2\mu_d}\sum_{i=1}^{n}\sum_{{\tau}=1}^{\T^S_i}|\nabla \td_{\tau}(\lambda_{\tau})|^2 = \frac{1}{2\mu_d}\sum_{i=1}^{n}\sum_{{\tau}=1}^{\T^S_i} z_{\tau}^2 
        \overset{(2)}{\leq} \frac{1}{2\mu_d}
        \sum_{i=1}^{n}\left(\sum_{{\tau}=1}^{\T^S_i} z_{\tau} \right)^2
        \overset{(3)}{\leq} \frac{1}{2\mu_d}\frac{\hlambda L^2_g}{\mu} \sum_{i=1}^{n}\sum_{{\tau}=1}^{\T^S_i} z_{\tau}\\
        &\overset{(4)}{\leq} 3\hlambda\left(\frac{L^2_g}{\mu\mu_d}\right)^2 \sum_{i=1}^{n}\sum_{\tau=1}^{\T^S_i+1}\hdelta_{\tau}
        \overset{(5)}{\leq} 3\hlambda\left(\frac{L^2_g}{\mu\mu_d}\right)^2 \sum_{i=1}^{n}\sum_{\tau=1}^{\T^S_i+1} \left(\delta + L_g\sqrt{\frac{2}{\mu} \left(\max_{x\in\mathcal{X}}|f_{\tau} - f_{\tau-1}| + \hlambda\delta\right)}\right)\\
        &\overset{(6)}{\leq} 3\hlambda\left(\frac{L^2_g}{\mu\mu_d}\right)^2 \sum_{t=1}^T \delta + 3\hlambda\left(\frac{L^2_g}{\mu\mu_d}\right)^2\sqrt{\frac{2}{\mu}}L_g \sum_{t=1}^T\left(\sqrt{ \max_{x\in\mathcal{X}}|f_{t} - f_{t-1}|} + \sqrt{\hlambda\delta}\right)\\
        &\overset{(7)}{\leq} 3\hlambda\left(\frac{L^2_g}{\mu\mu_d}\right)^2 \delta T + 3\hlambda\left(\frac{L^2_g}{\mu\mu_d}\right)^2\sqrt{\frac{2}{\mu}}L_g \left(\sqrt{T \sum_{t=1}^T\max_{x\in\mathcal{X}}|f_{\tau} - f_{\tau-1}|} + \sqrt{\hlambda\delta}T\right)\\
        &\overset{(8)}{\leq} 3\hlambda\left(\frac{L^2_g}{\mu\mu_d}\right)^2 V_{g,T} + 3\hlambda\left(\frac{L^2_g}{\mu\mu_d}\right)^2\sqrt{\frac{2}{\mu}}L_g \sqrt{T V_{f,T}} + 3\hlambda\left(\frac{L^2_g}{\mu\mu_d}\right)^2\sqrt{\frac{2\hlambda}{\mu}}L_g \sqrt{T V_{g,T}}\\
        &= \Ord\left(V_{g,T} + \sqrt{V_{f,T}T} + \sqrt{V_{g,T}T}\right) \\
        &\overset{(9)}{=} \Ord\left(\sqrt{V_{f,T}T} + \sqrt{V_{g,T}T}\right)
    \end{align*}
    where (1) is by Lemma \ref{lemma:local_str_conc_of_dual}, (2) is since $z_{\tau} \geq 0, \forall \tau\in[\T^S_i]$, (3) and (4) are by properties (A) and (B) in Lemma \ref{lemma:helpful_properties_for_case1}, respectively, (5) is by the definition of $\hdelta_t$ in Lemma \ref{lemma:slowly_changing_dual_gradients}, (6) is since $\forall X,Y \geq 0: \sqrt{X+Y} \leq \sqrt{X} + \sqrt{Y}$ and since $\sum_{i=1}^{n}\sum_{\tau=1}^{\T_i^S+1} \hdelta_{\tau} \leq \sum_{t=1}^T \hdelta_t$, (7) is by Jensen's inequality, (8) is by $V_{g,T}=\delta T$ and Assumption \ref{assum:bounded_TV_obj}, and (9) follows since $V_{g,T}=o(T)$.
    \vspace{-0.1cm}
    \paragraph{Putting it all together.} Combining the dual regret of all danger and safe phases, the total dual regret is bounded as follows:
    \begin{align}
        \Reg_{\td}(T)  &= \Reg^S_{\td} + \Reg^D_{\td}\\
        &= \Ord\left(V_{g,T} + V_{g,T} + \sqrt{V_{f,T}T} + \sqrt{V_{g,T}T}\right)\\
        &= \Ord\left(\sqrt{(V_{g,T} + V_{f,T})T}\right),
    \end{align}
    where the last equality follows since $V_{f,T}=o(T)$ and $V_{g,T}=o(T)$.
    Now, recall that the primal regret $R_f(T)$ can be bounded as follows by Lemma \ref{lemma:primal_dual_regret_relation}:
    \begin{equation}
        \Reg_f(T) = \Ord\left(\max\left\{ \hat{\Reg}_{\td}(T), \sqrt{(V_{g,T} + V_{f,T})T}\right\}\right).
    \end{equation}
    Thus, by plugging in the bound on the dual regret, we have:
    \begin{align}
       R_f(T) = \Ord\left(\sqrt{(V_{g,T} + V_{f,T})T}\right).
    \end{align}
\end{proof}

\newpage
\section{Extension to the Convex Case}\label{appendix:convex_case}
We extend our results to the convex case, namely where the loss functions are convex but \emph{not} necessarily strongly convex.
We show that in the convex case, Alg.~\ref{alg:naive} and Alg.~\ref{alg:safe_dual_alg_warm_start}, each with a slight modification, guarantee $\Ord\left(\left(V_{f,T}+V_{g,T}\right)^\frac{1}{3}T^\frac{2}{3}\right)$ and $\Ord\left(\left(V_{f,T}+V_{g,T}\right)^\frac{1}{7}T^\frac{6}{7}\right)$ regret, respectively.
Let $\{\hf_t\}_{t=1}^T$, where $\hf_t:\R^D\rightarrow \R, \forall t\in[T]$, be convex but \emph{not necessarily} strongly convex functions. We define the following surrogate functions:
\begin{equation}
    f_t(x) = \hf_t(x) + \frac{\mu}{2}\|x\|^2, \forall t\in[T].
\end{equation}
where $\mu>0$. Note that, by definition, $f_t$ is $\mu$-strongly convex, $\forall t\in[T]$.
For some decision sequence $\{x_t\}_{t=1}^T$, we define the regret in terms of the functions $\{\hf_t\}_{t=1}^T$ as follows:
\begin{equation}
    \Reg_{\hf}(T) = \sum_{t=1}^T \hf_t(x_t) - \hf_t(\hx_t^*),
\end{equation}
where the comparator sequence $\hx_t^*$ is defined as:
\begin{equation}
    \hx_t^* = \arg\min_{x\in\mathcal{X}} \hf_t(x) \quad\text{s.t.}\quad g_t(x)\leq 0.
\end{equation}
Namely, $\hx_t^*$ is the minimizer of the convex function $\hf_t(x)$ subject to the corresponding constraint $g_t(x)\leq0$. Note the contrast between $\hx_t^*$ and $x_t^* = \arg\min_{x\in\mathcal{X}} f_t(x) \;\text{s.t.}\; g_t(x)\leq 0$ which corresponds to the \emph{surrogate} functions. Now, we wish to bound the regret $\Reg_{\hf}(T)$ guaranteed by Alg.~\ref{alg:naive} and Alg.~\ref{alg:safe_dual_alg_warm_start}.

\begin{corollary}\label{corollary:guarantees_alg_naive_convex}
    Consider a safe online optimization problem of the form (\ref{eq:orignal_online_problem}) with horizon $T$. Running Alg.~\ref{alg:naive} or Alg.~\ref{alg:safe_dual_alg_warm_start} with the surrogate functions $f_t$ instead of $\hf_t$ guarantees zero constraint violation and $\Reg_{\hf}(T) = \Ord\left(\left(V_{f,T} + V_{g,T}\right)^\frac{1}{3} T^\frac{2}{3}\right)$ or $\Reg_{\hf}(T) = \Ord\left(\left(V_{f,T} + V_{g,T}\right)^\frac{1}{7} T^\frac{6}{7}\right)$, respectively.
\end{corollary}
\begin{proof}
    Both Alg.~\ref{alg:naive} and Alg.~\ref{alg:safe_dual_alg_warm_start} still guarantee zero constraint violation. The proof is identical to that of Theorem \ref{theorem:guarantees_naive_alg} and Theorem \ref{theorem:safe_stepsize_bounds} since we run Alg.~\ref{alg:naive} and Alg.~\ref{alg:safe_dual_alg_warm_start} over the surrogate functions $\{f_t\}_{t=1}^T$, while the constraints remain unchanged.
    
    Now, We show the regret guarantees for Alg.~\ref{alg:naive}.
    Note that by Theorem \ref{theorem:guarantees_naive_alg}, the regret in terms of the $\mu$-strongly convex surrogate functions $\{f_t\}_{t=1}^T$, which we denote $\Reg_{f}(T)$, is bounded as follows:
    \begin{equation}
        \Reg_{f}(T) \leq L_f R + \sqrt{\frac{2\hat{\lambda}}{\mu}}L_f \sqrt{V_{g,T}T} + \sqrt{\frac{2}{\mu}}L_f \sqrt{V_{f,T} T}.
    \end{equation}

    Also, note that $\Reg_{f}(T)$ (the regret in terms of the strongly convex surrogate functions $\{f_t\}_{t=1}^T$) can be related to $\Reg_{\hf}(T)$ (the regret in terms of the convex functions $\{\hf_t\}_{t=1}^T$) as follows:
    \begin{align}
        \Reg_f(T) &= \sum_{t=1}^T f_t(x_t) - f_t(x_t^*) \\
        &\overset{(1)}{\geq} \sum_{t=1}^T f_t(x_t) - f_t(\hx_t^*) \\
        &= \sum_{t=1}^T \left(\hf_t(x_t) - \hf_t(\hx_t^*)\right) + \sum_{t=1}^T \frac{\mu}{2}\left(\|x_t\|^2 - \|\hx_t^*\|^2\right) \\
        &= \Reg_{\hf}(T)+ \sum_{t=1}^T \frac{\mu}{2}\left(\|x_t\|^2 - \|\hx_t^*\|^2\right),
    \end{align}
    where (1) follows since $x_t^* = \arg\min_{x\in\mathcal{X}} f_t(x) \;\text{s.t.}\; g_t(x)\leq0$ and thus $f_t(x_t^*) \leq f_t(x), \forall x: g_t(x)\leq0$.
    Thus, the regret in terms of the convex functions $\{\hf_t\}_{t=1}^T$, which we denote $\Reg_{\hf}(T)$, is bounded as follows:
    \begin{align}
        \Reg_{\hf}(T) &\leq \Reg_f(T) + \frac{\mu}{2}\sum_{t=1}^T \left(\|\hx_t^*\|^2 - \|x_t\|^2\right)\\
        &\leq \Reg_f(T) + \frac{\mu}{2}\sum_{t=1}^T \|\hx_t^*\|^2 \\
        &\overset{(1)}{\leq} \Reg_f(T) + \frac{\mu}{2}R^2T \\
        &\overset{(2)}{\leq} L_f R + \sqrt{\frac{2\hat{\lambda}}{\mu}}L_f \sqrt{V_{g,T}T} + \sqrt{\frac{2}{\mu}}L_f \sqrt{V_{f,T} T} + \frac{\mu}{2}R^2T \\
        &= L_f R + \sqrt{2}L_f\frac{\sqrt{\hlambda}\sqrt{V_{g,T}T}+\sqrt{V_{f,T}T}}{\sqrt{\mu}} + \frac{\mu}{2}R^2T
    \end{align}
    where (1) follows by Assumption \ref{assum:bounded_set} (bounded set) and (2) follows by Theorem \ref{theorem:guarantees_naive_alg}. Note that this bound holds for any $\mu>0$. Thus, optimizing over $\mu$ yields $\mu^* \propto (V_{f,T}^\frac{1}{3} + V_{g,T}^\frac{1}{3})T^{-\frac{1}{3}}$, and plugging $\mu=(V_{f,T}^\frac{1}{3} + V_{g,T}^\frac{1}{3})T^{-\frac{1}{3}}$ back in the bound yields:
    \begin{align}
        R_{\hf}(T) &\leq L_f R + \sqrt{2}L_f\frac{\sqrt{\hlambda}\sqrt{V_{g,T}T}+\sqrt{V_{f,T}T}}{\sqrt{(V_{f,T}^\frac{1}{3} + V_{g,T}^\frac{1}{3})T^{-\frac{1}{3}}}} + \frac{(V_{f,T}^\frac{1}{3} + V_{g,T}^\frac{1}{3})T^{-\frac{1}{3}}}{2}R^2T \\
        &\leq L_f R + \sqrt{2}L_f\left(\frac{\sqrt{\hlambda}\sqrt{V_{g,T}T}}{\sqrt{V_{g,T}^\frac{1}{3}T^{-\frac{1}{3}}}} + \frac{\sqrt{V_{f,T}T}}{\sqrt{V_{f,T}^\frac{1}{3}T^{-\frac{1}{3}}}}\right) + \frac{(V_{f,T}^\frac{1}{3} + V_{g,T}^\frac{1}{3})T^{-\frac{1}{3}}}{2}R^2T \\
        &\leq L_f R + \sqrt{2}L_f\left(\sqrt{\hlambda}V_{g,T}^\frac{1}{3}T^\frac{2}{3} + V_{f,T}^\frac{1}{3}T^\frac{2}{3}\right) + \frac{R^2}{2}(V_{f,T}^\frac{1}{3} + V_{g,T}^\frac{1}{3})T^\frac{2}{3} \\
        &=\Ord\left(\left(V_{f,T}+V_{g,T}\right)^\frac{1}{3} T^\frac{2}{3}\right)
    \end{align}
    As stated. Now, proving the regret guarantees for Alg.~\ref{alg:safe_dual_alg_warm_start} follows the same lines, but now we use the bound given by Theorem \ref{theorem:regret_analyis_our_method} for $\Reg_f(T)$. Namely:
    \begin{align}
        \Reg_{\hf}(T) &\leq \Reg_f(T) + \frac{\mu}{2}R^2T \\
        &\leq \frac{6}{\mu_d}\left(\delta V_{g,T} + 2L_g\sqrt{\frac{2}{\mu}}\delta\sqrt{TV_{f,T}} + 2L_g\sqrt{\frac{2\hlambda}{\mu}}\delta\sqrt{TV_{g,T}} + L^2_g\frac{2}{\mu}V_{f,T} + L^2_g\frac{2}{\mu}\hlambda V_{g,T} \right) + \\
        &\quad\quad + 3\hlambda\left(\frac{L^2_g}{\mu\mu_d}\right)^2 V_{g,T} + 3\hlambda\left(\frac{L^2_g}{\mu\mu_d}\right)^2\sqrt{\frac{2}{\mu}}L_g \sqrt{T V_{f,T}} + 3\hlambda\left(\frac{L^2_g}{\mu\mu_d}\right)^2\sqrt{\frac{2\hlambda}{\mu}}L_g \sqrt{T V_{g,T}} + \frac{\mu}{2}R^2T
    \end{align}
    By similarly optimizing over $\mu$ and substituting $\mu=\left(V_{f,T}+V_{g,T}\right)^\frac{1}{7}T^{-\frac{1}{7}}$ we have:
    \begin{equation}
        \Reg_{\hf}(T) \leq \Ord\left(\left(V_{f,T} + V_{g,T}\right)^\frac{1}{7} T^\frac{6}{7}\right).
    \end{equation}
\end{proof}

\end{document}